\documentclass[sigconf,nonacm]{acmart}
\AtBeginDocument{%
  }

\usepackage{graphicx}
\usepackage{amsmath}
\usepackage{amsthm}
\usepackage{booktabs}
\usepackage{algorithm}
\usepackage{algorithmic}
\usepackage[switch]{lineno}
\usepackage{bm}
\usepackage{subcaption}

\usepackage{multirow}
\usepackage{wrapfig}
\def\Vbar{\perp\!\!\!\perp}
\def\NVbar{\not \! \perp \!\!\! \perp}

\newtheorem{assumption}{Assumption}
\newtheorem{theorem}{Theorem}

\newtheorem{proposition}{Proposition}
\newtheorem{definition}{Definition}

\usepackage{enumitem}
\newtheorem{remark}{Remark}

\newenvironment{rep_theorem}[1]
  {\renewcommand{\thetheorem}{\ref{#1}}%
   \begin{theorem}}
  {\end{theorem}%
   \addtocounter{theorem}{-1}}
   
\def\bR{{\mathbb R}}

\def\bE{{\mathbb E}}

\def\cA{{\mathcal A}}
\def\cX{{\mathcal X}}
\def\cU{{\mathcal U}}
\def\cS{{\mathcal S}}
\def\cY{{\mathcal Y}}

\def\cZ{{\mathcal Z}}
\def\cE{{\mathcal E}}
\def\cM{{\mathcal M}}

\def\b{\mathbf{b}}
\def\w{\mathbf{w}}
\def\x{\mathbf{x}}
\def\u{\mathbf{u}}
\def\s{\mathbf{s}}
\def\z{\mathbf{z}}
\def\r{\mathbf{r}}
\def\X{\mathbf{X}}
\def\U{\mathbf{U}}
\def\S{\mathbf{S}}
\def\C{\mathbf{C}}
\def\Z{\mathbf{Z}}
\def\R{\mathbf{R}}
\def\W{\mathbf{W}}
\def\b{\mathbf{b}}
\def\V{\mathbf{V}}

\setcopyright{none}
\begin{document}


\title{
Long-term Causal Inference with Continuous Treatment: Generalization Bound via Optimal Transport Weights
}

\author{Zeqin Yang}
\authornote{These authors contributed equally to this work.}
\affiliation{%
  \institution{Guangdong University of Technology}
  \city{Guangzhou}
  \country{China}}
\email{youngzeqin@gmail.com}

\author{Weilin Chen}
\authornotemark[1]
\affiliation{%
  \institution{Guangdong University of Technology}
  \city{Guangzhou}
  \country{China}}
\email{chenweilin.chn@gmail.com}

\author{Ruichu Cai}
\authornote{Corresponding author.}
\affiliation{%
  \institution{Guangdong University of Technology}
  \city{Guangzhou}
  \country{China}}
\email{cairuichu@gmail.com}

\author{Yuguang Yan}
\affiliation{%
  \institution{Guangdong University of Technology}
  \city{Guangzhou}
  \country{China}}
\email{ygyan@gdut.edu.cn}

\author{Zhifeng Hao}
\affiliation{%
  \institution{Shantou University}
  \city{Shantou}
  \country{China}}
\email{haozhifeng@stu.edu.cn}

\author{Zhipeng Yu}
\affiliation{%
  \institution{Didi Chuxing}
  \city{Beijing}
  \country{China}}
\email{yuzhipeng@didiglobal.com}

\author{Zhichao Zou}
\affiliation{%
  \institution{Didi Chuxing}
  \city{Beijing}
  \country{China}}
\email{zouzhichao@didiglobal.com}

\author{Jixing Xu}
\affiliation{%
  \institution{Didi Chuxing}
  \city{Beijing}
  \country{China}}
\email{stevenxujixing@didiglobal.com}

\author{Peng Zhen}
\affiliation{%
  \institution{Didi Chuxing}
  \city{Beijing}
  \country{China}}
\email{zhenpeng@didiglobal.com}

\author{Jiecheng Guo}
\affiliation{%
  \institution{Didi Chuxing}
  \city{Beijing}
  \country{China}}
\email{jasonguo@didiglobal.com}

\renewcommand{\shortauthors}{Yang et al.}


\begin{abstract}
Estimating long-term treatment effect is difficult 
since long-term outcomes are typically available only in observational data 
where unobserved confounders biases causal estimation, 
while experimental data is often limited to short-term outcomes. 
The problem is further complicated by continuous treatments and individual-level heterogeneity, which require recovering an entire personalized response curve rather than a single average treatment effect.
To address these challenges,
we formulate a framework that combines these complementary data sources to estimate long-term Heterogeneous Dose-Response Curve (HDRC).
First, 
to remove the unobserved confounders in the long-term observational data,
we construct an optimal transport-based reweighting scheme that aligns the long-term observational data to an auxiliary short-term experimental data, thereby identifying the long-term HDRC. 
Second, under the induced weighted distribution, we analyze counterfactual prediction with continuous treatments and derive a generalization bound for long-term estimation error. 
Guided by the resulting theory, we develop a practical long-term HDRC estimator. 
Empirical studies across synthetic and semi-synthetic datasets highlight the effectiveness of the proposed method.

\end{abstract}



\maketitle

\section{Introduction}

Long-term  treatment effect estimation is practically significant in various domains \cite{hohnhold2015focusing,fleming1994surrogate,chetty2011does}.
While Randomized Controlled Trials (RCTs) are the gold standard for estimating causal effects, collecting long-term outcomes through RCTs is often infeasible due to high costs, e.g., it may take years or even decades to collect mortality outcomes in clinical trials.
As a result, only short-term experimental data are typically available.
In contrast, long-term observational data are often more accessible and cost-effective, and thus are commonly taken as a complement to short-term experimental data.

Under such data combination setting,
there has been growing interest in leveraging short-term outcomes to assist in inferring long-term effects.
With all observed confounders, 
to estimate long-term average effects,
\cite{kallus2020role} studies the efficiency gains of incorporating the short-term outcomes from experimental data with only limited long-term outcomes under binary treatments, while \cite{zeng2024continuous} extends this to continuous treatments.
As a counterpart, allowing for unobserved confounders, \cite{athey2020combining,ghassami2022combining,hu2025identification,imbens2022long} propose different assumptions to use short-term outcomes from experimental data to remove confounding in long-term observational data.
Nevertheless, they still focus on long-term average effects of binary treatment.

Therefore, previous works lack practical applicability, 
since they have largely focused on estimating long-term average effects under ideal assumptions, e.g.,
no unobserved confounders in observational data or binary treatment setting.
In practice,
to develop personalized policies, decision-makers usually require long-term individual effects estimated with unobserved confounders and continuous treatment,
where the average treatment effects are insufficiently informative and the above assumptions are easily violated.
Consider a vocational training example where a company wants to evaluate the effect of training hours (\textit{continuous treatment}) on one-year employment income (\textit{long-term outcome}) of different participants (\textit{heterogeneity}). 
To support this, two types of data can be collected:
short-term experimental data that only collect multiple post-training skill assessment scores (\textit{short-term outcomes}),
and long-term historical observational data that contain both short- and long-term outcomes but suffer from unobserved confounders, such as participants’ learning aptitude.
In such a scenario, how to estimate the individual treatment effect of training hours on the long-term income of a specific participant, i.e., long-term Heterogeneous Dose-Response Curve (HDRC), still remains challenging.

Motivated by the example, this paper focuses on estimating long-term HDRC with unobserved confounders and continuous treatment via data combination,
which presents two key challenges.
First, the long-term outcome only exists in the observational data, suffering from unobserved confounders, and is missing in the experimental data, causing the identification problem. 
Second, the spaces of short- and long-term counterfactual outcomes are infinite since the treatment is continuous, causing large counterfactual prediction error of the long-term HDRC with only the actual outcome.

To address the first challenge of identifiability, 
we theoretically demonstrate that 
unobserved confounders can be eliminated by aligning the conditional distributions of short-term outcomes between observational and experimental data.
To achieve such alignment, under the Optimal Transport (OT) reweighting framework, we establish a mini-batch and joint distribution-based upper bounds on the discrepancy of conditional distributions, 
which is computationally efficient and can be easily embedded into deep learning.
Further, to address the second challenge of lacking counterfactual outcomes, 
we derive a generalization bound of counterfactual regression based on the reweighted distribution induced by OT, 
showing the counterfactual prediction error of the long-term HDRC can be bounded by the observed weighted factual error with a weighted IPM term.
Based on the above theoretical results, 
we propose our model called \textbf{L}ong-term h\textbf{E}terogeneous dose-response curve estim\textbf{A}tor with \textbf{R}eweighting and represe\textbf{N}tation learning (LEARN),
capable of reducing the observed and unobserved confounding bias.

Our contributions are summarized as follows. 
\textit{First}, to the best of our knowledge, this is the first work to address the problem of estimating long-term HDRC with unobserved confounders and continuous treatment via data combination.
\textit{Second}, we theoretically show that the unobserved confounders can be removed by aligning the conditional distributions between observational and experimental data, 
and propose a practical and computationally OT weighting method to achieve such alignment.
\textit{Third}, based on the OT-induced reweighted distribution, we derive the generalization bound on the counterfactual prediction error in long-term HDRC estimation, which inspires our model. 
\textit{Finally}, we conduct extensive experiments on multiple synthetic and semi-synthetic datasets to demonstrate the effectiveness of our proposed method.


\section{Related Work}

\textbf{Long-term Causal Inference.}
A common strategy for long-term causal inference is to leverage short-term outcomes via data combination.
Under the surrogacy assumption that short-term outcomes fully mediate the effect of treatment,
\cite{athey2019surrogate} combines multiple short-term outcomes to robustly estimate the long-term average effects,
\cite{cheng2021long} models the relationship between the short- and long-term outcomes via deep learning,
\cite{singh2022generalized} proposes kernel estimator for more complex average causal estimands without linearity and separability assumptions.
However,
the surrogacy assumption is often too strong in practice.
With all observed confounders,
\cite{kallus2020role} studies the role of short-term outcomes for long-term causal inference under binary treatment, 
while \cite{zeng2024continuous} serves as a counterpart for continuous treatment.
\cite{cai2024long} consider settings with partially unobserved surrogates and employ identifiable variational autoencoder for recovery. 
\cite{traninferring} address dynamic treatments through reinforcement learning,
and \cite{yang2024learning,wu2024policy} develop policy learning method for balancing short- and long-term rewards.
Further,
to address the unobserved confounders,
\cite{athey2020combining} introduce relaxed latent unconfoundedness for identification, 
which is further leveraged by \cite{chen2023semiparametric} to develop a doubly robust estimator.
Alternative identifiability assumptions are proposed by 
\cite{ghassami2022combining,hu2025identification,obradovic2024identification}.
Interestingly, \cite{park2024bracketing} studies the bracketing relationship between identifiability assumptions in \cite{athey2020combining} and \cite{ghassami2022combining}.
Based on proximal methods\cite{tchetgen2024introduction},
\cite{imbens2022long} exploits the sequential structure of short-term outcomes.
\cite{van2023estimating} performs instrumental variable regression to estimate long-term effects.
Despite these advances across various settings,
to the best of our knowledge,
we are the first to focus on the long-term HDRC estimation problem with unobserved confounders and continuous treatment.

\textbf{Dose-Response Curve Estimation.}
Dose-response curve estimation is broadly studied through statistical methods \cite{vegetabile2021nonparametric,kennedy2017non,zhu2015boosting},
and most of them are based on the generalized propensity score \cite{imbens2000role} or entropy balancing \cite{tubbicke2022entropy}.
Recently, 
deep learning-based methods also attract the attention of the research community.
DRNet \cite{schwab2020learning} divides the continuous treatment into several intervals, and trains one separate head for each interval.
In order to produce a continuous curve,
VCNet \cite{nie2021vcnet} adapts a varying coefficient model into neural networks.
SCIGAN \cite{bica2020estimating} generates counterfactual outcomes based on the generative adversarial network.
ADMIT \cite{wang2022generalization} derives a counterfactual bound of estimating the average dose-response curve with a discretized approximation of the IPM distance,
and ACFR \cite{kazemi2024adversarially} minimizes the KL-divergence using an adversarial game to extract balanced representations for continuous treatment.
Different from the existing works above,
we aim to estimate the HDRC under the long-term estimation scenario with unobserved confounders in this paper, which is more complex.

\textbf{Optimal Transport.} 
Optimal transport seeks a minimum-cost plan for moving mass between distributions
\cite{1781memoireMonge,1958translocationKantorovitch,2008optimalVillani},
and has recently been widely used in various applications
\cite{peyre2019computational,2019oversamplingYan,zhao2018label}.
There are some works trying to introduce optimal transport into causal inference.
\cite{gunsilius2021matching} employs unbalanced optimal transport
for matching.
\cite{li2021causal} infers counterfactual outcomes via transporting the factual distribution to the counterfactual one.
\cite{dunipace2021optimal,yanreducing,yan2024exploiting} learn balancing weights via optimal transport.
\cite{wang2024optimal} use optimal transport to address mini-batch sampling and unobserved confounders under CFR \cite{shalit2017estimating} framework.
To the best of our knowledge,
this is the first work exploiting optimal transport to solve the unobserved confounder problem via data combination for long-term effect estimation.

\section{Preliminary}
\subsection{Optimal Transport Weighting}
Optimal Transport (OT) \cite{villani2021topics} is widely used to quantify distribution discrepancy as the minimum cost of transporting one distribution to another.
Among the rich theory of OT, we focus on the \textit{Kantorovich Problem} \cite{kantorovich2006translocation} for discrete distribution in this paper.
Recently,
thanks to its many advantages, 
such as geometric sensitivity and stability,
OT has been explored for learning weights to align distribution \cite{dunipace2021optimal,guo2022learning,yan2024exploiting}.
Specifically,
suppose we want to learn weights $\w_{\alpha}$ for samples in distribution $\alpha$, 
such that the reweighted distribution is aligned with another distribution $\beta$.
To achieve this goal,
we view the to-be-learned weights $\w_{\alpha}$ as the probability measure of distribution $\alpha$,
and represent the distribution $\beta$ with a uniform probability measure $\bm{\mu}$.
By doing so, 
the learning of $\w_{\alpha}$ can be formulated as the process of minimizing the OT distance between distributions $\alpha$ and $\beta$:
\begin{small}
    \begin{align}
    \label{eq: ot weight}
        &\min_{\w_{\alpha}} OT(\alpha, \beta)=\min_{\w_{\alpha}}\min_{\Gamma\in \Pi(\alpha, \beta)} \langle \Gamma, C \rangle, \nonumber\\
        &s.t.~~
        \Pi(\alpha, \beta)=\{\Gamma\in \bR^{n\times m}|\Gamma \mathbf{1}_m=\w_{\alpha}, \Gamma^T \mathbf{1}_n=\bm{\mu}, \Gamma_{ij}\in \left[0, 1\right]\},
    \end{align}
\end{small}
where $\langle \cdot, \cdot \rangle$ is the Frobenius inner product,
$n$ and $m$ are the sample numbers of $\alpha$ and $\beta$,
the transport cost matrix $C$ is the unit-wise distance between $\alpha$ and $\beta$, 
the transport probability matrix $\Gamma\in \Pi(\alpha, \beta)$ is learned by minimizing $OT(\alpha, \beta)$,
which reflects how to transport samples from $\alpha$ to $\beta$.

\subsection{Long-term HDRC Estimation}

Let 
$A\in \cA$ be the 1-dimensional continuous treatment, 
$\X\in \cX$ be the observed confounders,
$\U\in \cU$ be the unobserved confounders,
$\S(a)\in \cS$ be the potential short-term outcomes measured at timesteps $1,2,...,t_0$,
and $Y(a)\in \cY$ be the potential long-term outcome at timestep $T$.
Let lowercase letters (e.g., $a, \x, \u, \s(a),y(a)$) denote the value of the random variables. 
Following the potential outcome framework \cite{rubin1974estimating},
the observed short- and long-term outcomes $\S, Y$ are the potential outcomes $\S(a), Y(a)$  corresponding to the actually treatment $a$. 
Besides, we use $P, \hat{P}$ to represent the population and empirical distributions, respectively.

Following \cite{athey2020combining,chen2023semiparametric,ghassami2022combining},
we consider a data combination setting where $G= \{o, e\}$ distinguishes between two types of data:
a large observational data $O=\{a_i, \x_i, \s_i, y_i, G_i=o\}_{i=1}^{n_o}$
and
a small experimental data $E=\{a_i, \x_i, \s_i, G_i=e\}_{i=n_o +1}^{n_o + n_e}$ ($n_e\ll n_o$).
That is, 
the treatment $A$, covariates $\X$ and short-term outcomes $\S$ are available in both datasets, 
while long-term outcome $Y$ is only available in observational data.
Our goal is to estimate the long-term Heterogeneous Dose-Response Curve (HDRC), 
which is defined as:
\begin{small}
    \begin{align}
    \mu(\x, a) = \bE[Y(a)|\X=\x].
\end{align}
\end{small}
Notably,
the long-term HDRC cannot be estimated solely from the experimental data due to the missingness of long-term outcome $Y$,
and
also cannot be estimated solely from the observational data due to the possible unobserved confounders. 
To overcome these challenges,
we learn long-term HDRC via data combination with the following assumptions:

\begin{assumption} [Identification Assumptions] 
    \label{assum: long term total} 
    ~
    \begin{enumerate}[label=\textbf{A\arabic*}, ref=A\arabic*]
        \item \textbf{[Consistency]} \label{assum: consist}
        If $A=a$, then  $Y=Y(a),~ \S=\S(a)$. 
        \item  \textbf{[Positivity]} \label{assum: positi}
        $\forall a,\x$, we have
        $0<P(A=a|\X=\x)<1$ and $0<P(G=o|A=a, \X=\x)<1$.
        \item  \textbf{[Weak Internal validity of observational data]}
        \label{assum: internal validity of obs}
        \newline
        $A\Vbar \{Y(a),\S(a)\}|\X, \U, G=o$; $A\NVbar \{Y(a),\S(a)\}|\X, G=o$.
        \item  \textbf{[Internal validity of experimental data]}  \label{assum: internal validity of exp}
        \newline
        $A\Vbar \{Y(a),\S(a)\}|\X, G=e$.
        \item \textbf{[Unconfounded selection]} \label{assum: external validity of exp}
        $G\Vbar \{Y(a),\S(a)\}|\X$.
        \item \textbf{[Latent Unconfoundedness]}  \label{assum: LU}
        $A\Vbar Y(a)|\X, \S(a), G=o$. 
    \end{enumerate}
\end{assumption}







The assumptions above are mild and widely used, e.g.,  \cite{athey2020combining,ghassami2022combining}.
Assumption \ref{assum: consist} and \ref{assum: positi} are standard assumptions in causal inference.
Assumption \ref{assum: internal validity of obs} allows unobserved confounders to exist in observational data,
which is much weaker than the classic unconfoundedness assumption. 
However, this assumption also poses challenges for identification 
Different from Assumption \ref{assum: internal validity of obs},
Assumption \ref{assum: internal validity of exp} does not allow for the presence of unobserved confounders in experimental data.
Assumption \ref{assum: external validity of exp} 
allows us to generalize the distribution of potential outcomes between different types of data $G$,
and this assumption is reasonable because the difference between experimental and observational data is usually the treatment mechanism, while the potential outcome distributions remain unchanged and stable.
Finally,
Assumption \ref{assum: LU} is the key assumption in our paper,
which is first proposed by \cite{athey2020combining} to achieve long-term average effect identification under the binary-treatment setting.
Different from \cite{athey2020combining}, we reuse this assumption to identify long-term HDRC by a novel reweighting schema.
This assumption means that unobserved confounders in the observational data is mediated through the short-term outcomes,
effectively blocking the path $\U \rightarrow Y$.
In our vocational training example,
this assumption means that the short-term skill assessment scores fully reflect the participants’ learning aptitude.
The causal graphs that satisfy the above assumptions are shown in Fig. \ref{fig: graphs}.

\begin{remark}
The plausibility of Assumption \ref{assum: LU} can be assessed by sensitivity analysis.
\cite{kallus2019interval, frauen2023sharp} provide functional interval method under marginal sensitivity models.
Since Assumption~6, $A\perp Y(a)\mid X,S(a)$, only augments standard unconfoundedness $A\perp Y(a)\mid X$ by conditioning on $S(a)$, so these methods can be naturally adapted here by treating $S$ as an additional conditioning variable.
\end{remark}

\begin{figure}[!t]
    \centering
    \includegraphics[width=.65\linewidth]{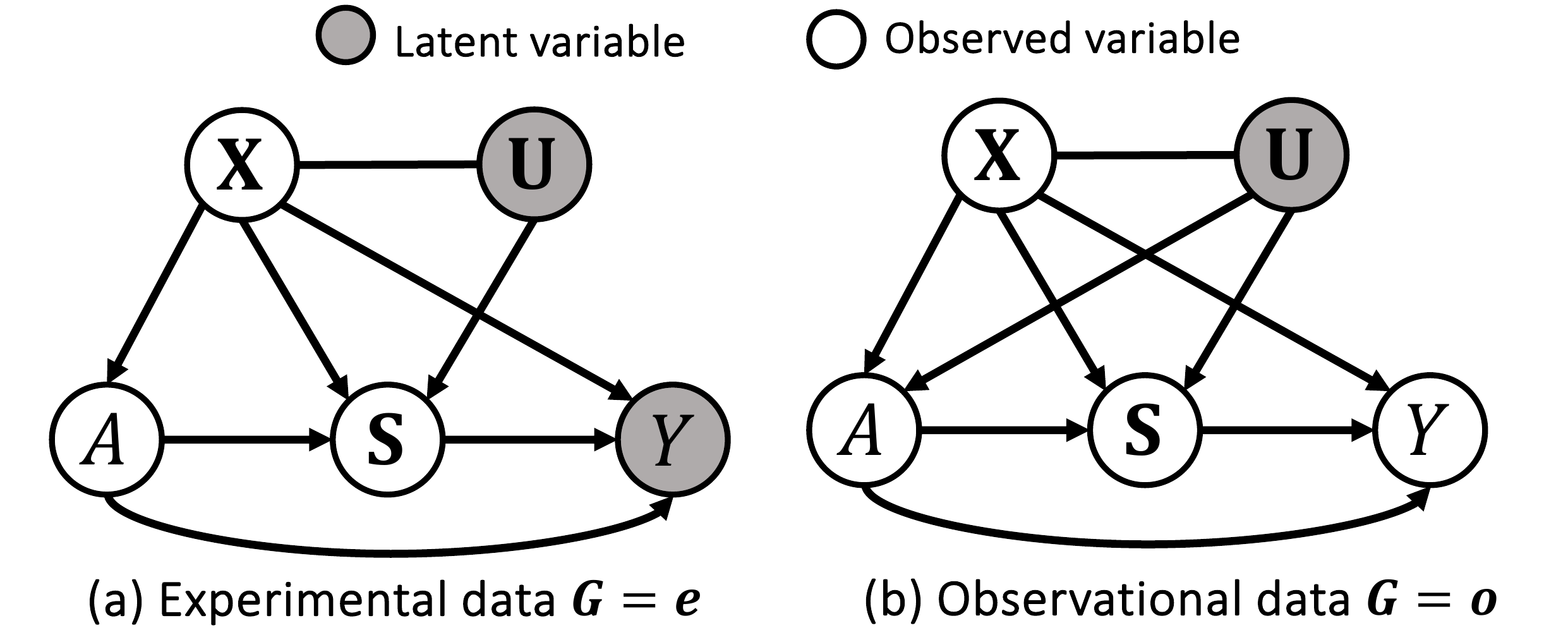}
    \caption{Causal graphs of two types of data.}
    \label{fig: graphs}
\vspace{-0.35cm}
\end{figure}

\section{Proposed Method}
To estimate long-term HDRC unbiasedly, 
we address unobserved confounders via reweighting and observed confounders via representation learning, as detailed in Sections \ref{sec: reweight} and \ref{sec: repre}, respectively.

\subsection{Optimal Transport Weights for Unobserved Confounders via Data Combination}
\label{sec: reweight}
We now present our method for handling the unobserved confounders in observational data.
Firstly,
we propose a novel reweighting schema dealing with the unobserved confounders in observational data, making long-term HDRC identifiable. 
Then we provide a practical and computationally OT method to learn such weights via data combination.

\subsubsection{Identifiable Long-term HDRC via Reweighting}
\label{sec: identifiable}

Since $Y$ is only observed in observational data, 
identifying long-term HDRC requires the unconfoundedness assumption of observational data, i.e., $A \Vbar \{\S(a), Y(a)\}|\X, G=o$,
which does not hold due to Assumption \ref{assum: internal validity of obs}.
Fortunately, $A \Vbar \{\S(a), Y(a)\}|\X, G=e$ holds, 
though $Y$ is missing in the experimental data.
It raises a question: 
can we remove the influence of unobserved confounders in the observational data to achieve identification, 
by utilizing experimental data that only includes short-term outcomes? 
We find that, 
although $\U$ is unobserved, its influence can be captured by the (in)dependence between $\S$ and $G$ conditional on $\X$ and $A$, 
as illustrated in the Proposition \ref{pro: reweight unconf}.


\begin{proposition}
\label{pro: reweight unconf}
    Suppose Assumptions~\ref{assum: long term total} holds, and
    consider weights $\w = \{\w_o, \bm\mu\}$, where 
    $\w_o$ are learnable weights for observational units ($G=o$) and 
    $\bm\mu$ are uniform constant weights for experimental units ($G=e$).
    Let $P^\w$ denote the induced reweighted distribution,
    under which we assume the following conditional independence holds:
    \begin{small}
        \begin{align}
        P^\w(\S,G \mid \X, A) 
        = P^\w(\S \mid \X, A)\, P^\w(G \mid \X, A),
        \label{eq:SwGXA}
    \end{align}
    \end{small}
    i.e., $ \S \Vbar G \mid \X,A$.
    Then we have:
    \begin{small}
        \begin{align}
        P^\w\big(\S(a)\mid \X=\x, A=a, G=o\big)
        = P^\w\big(\S(a)\mid \X=\x, G=o\big). \nonumber
    \end{align}
    \end{small}
\end{proposition}

The proof is given in Appendix \ref{proof: reweight unconf}. Proposition \ref{pro: reweight unconf} is intuitive. Since unconfoundedness holds in $G=e$, 
then $P(\S(a)|\X,G=e)=P(\S|\X,A=a,G=e)$, but it does not hold when $G=o$.
If we can learn weight $\w$ aligning short-term outcomes between observational data and experimental data, 
i.e., making $P(\S|\X, A, G=e)=P(\S|\X, A, G=o)$, 
the unconfoundedness will also hold in $G=o$ based on Assumption \ref{assum: external validity of exp}. 
Based on Proposition \ref{pro: reweight unconf}, we further show the long-term HDRC is identifiable as follows.


\begin{theorem}
\label{thm:identification}
    Suppose Assumptions~\ref{assum: long term total} holds,
    and consider weights $\w=\{\w_o,\bm\mu\}$ such that the conditions of Proposition~\ref{pro: reweight unconf} are satisfied.
    Let $P^\w$ denote the corresponding reweighted distribution.
    Then, for all $\x,a$, the following weighted (mean) unconfoundedness for the long-term potential outcome holds:
    \begin{small}
        \begin{align}
        \mathbb E_{P^\w}\big[Y(a)\mid \X=\x, A=a, G=o\big]
        = \mathbb E_{P^\w}\big[Y(a)\mid \X=\x, G=o\big], 
        \label{eq:Ya-unconf-weighted}
    \end{align}
    \end{small}
    and the long-term heterogeneous dose-response curve (HDRC)
    \(
        \mu(\x,a) = \mathbb E\big[Y(a)\mid \X=\x\big]
    \)
    is identified by
    \begin{small}
        \begin{align}
        \mu(\x,a)
        = \mathbb E\big[Y(a)\mid \X=\x\big]
        = \mathbb E_{P^\w}\big[Y \mid \X=\x, A=a, G=o\big]. 
        \label{eq:hdrci-formula}
        \end{align}
    \end{small}
\end{theorem}

The proof is given in Appendix \ref{proof: iden}. 
Theorem \ref{thm:identification} is built on Proposition \ref{pro: reweight unconf}, indicating that if we could learn suitable weights making $\S \Vbar G|\X,A$,
the influence of unobserved confounders will be removed,
and then long-term HDRC can be identified.
Therefore,
our goal now is to learn a set of weights $\w_o$ for observational samples such that  $\S \Vbar G|\X,A$ holds based on the reweighted distribution,
i.e.,
$P^{\w_o}(\S|\X, A, G=o)=P^{\bm{\mu}}(\S|\X, A, G=e)$.

\subsubsection{Learning Optimal Transport Weights}
\label{sec: learning}

Based on the above theoretical analyses,
we need to learn weights $\hat{\w}_o$ for observational samples to minimize the distances between the empirical distributions $\hat{P}(\S|\X, A, G=o)$ and $\hat{P}(\S|\X, A, G=e)$.
However, most existing works only study how to measure the distances between joint distributions, instead of conditional distributions. 
Surprisingly, we find that under the optimal transport framework, the distance between conditional distributions can be bounded by that of joint distributions.
As a result, we propose a mini-batch joint distribution-based OT method that not only helps align the conditional distributions but also can be easily embedded into the training of deep learning with theoretical guarantees.

Following Eq.(\ref{eq: ot weight}),
our goal is estimating weights $\hat{\w}_o$ for observational samples via minimizing the \textit{conditional OT distance},
i.e.,
$OT(\hat{P}(\S|\X, A, G=o), \hat{P}(\S|\X, A, G=e))$.
The intuitive approach to solve this problem is to minimize every sub-problem for every realization $\X=\x, A=a$ as follows:
$OT^{con}_{\X, A} = \sum_{\x, a}OT(\hat{P}(\S|\X=\x, A=a, G=o), \;\hat{P}(\S|\X=\x, A=a, G=e)) \hat{P}(\x, a)$.


\textbf{The First Challenge.}
However,
solving the above conditional OT problem is infeasible.
The reason is that the conditional set $(\X, A)$ are multi-dimensional continuous variables, 
leading to the existence of infinitely sub-problems.
To overcome this challenge,
we establish a connection between the OT distance regarding conditional distributions and that regarding joint distributions:
\begin{theorem}
\label{theorem: ot1}
    Under mild conditions in Appendix \ref{proof: ot1}, 
    we have:
    \begin{small}
        \begin{align}
        OT^{con}_{\x, a}\leq OT(\hat{P}(\S, \X, A|G=o), \hat{P}(\S, \X, A|G=e)).
        \end{align}
    \end{small}
\end{theorem}

The proof is given in Appendix \ref{proof: ot1}.
Theorem \ref{theorem: ot1} shows that the conditional OT distance can be bounded by the OT distance between the corresponding joint distributions,
which motivates us to learn weights $\hat{\w}_o$ by minimizing the upper bound $OT(\hat{P}(\S, \X, A|G=o), \hat{P}(\S, \X, A|G=e))$ instead of the infeasible $OT^{con}_{\x, a}$.

\textbf{The Second Challenge.}
While we solve the problem of conditional OT, 
another challenge about intensive computation emerges.
Specifically,
solving $OT(\hat{P}(\S, \X, A|G=o), \hat{P}(\S, \X, A|G=e))$ will cause a high computational cost since the whole optimized transport matrix $\Gamma\in \bR^{n_o\times n_e}$ can be large with the large size of whole observational data $n_o$.
Additionally, optimizing the whole matrix at each iteration is also not suitable for the mini-batch training manner of deep learning.
To tackle this challenge,
inspired by \cite{yang2023prototypical},
we propose a mini-batch OT based on randomly sub-sampled observational data and provide its theoretical analysis as follows:
\begin{definition}
    Consider two empirical distributions $\alpha$ and $\beta$ with $n$ and $m$ units, 
    we assume the batch size $b$ of the first distribution $\alpha$ satisfy $b\mid n$ and let $ k=n/b$ be the number of batches.
    Let $\mathcal{B}_{i}$ be index set of the $i$-th batch  in $\alpha$ and the corresponding empirical distribution is $\alpha_{\mathcal{B}_i}$,
    then the mini-batch OT problem is defined as:
    \begin{small}
            \begin{align}
        m\mbox{-}OT(\alpha, \beta)=\frac{1}{k}\sum_{i=1}^k OT(\alpha_{\mathcal{B}_i}, \beta).
    \end{align}
    \end{small}
\end{definition}
\begin{theorem}
\label{theorem: ot2}
    Let $\gamma_i$ be the optimal transport probability matrix of the $i$-th batch OT problem of $m\mbox{-}OT(\hat{P}(\S, \X, A|G=o), \hat{P}(\S, \X, A|G=e))$.
    With extending $ \gamma_i$ to a $n_o\times n_e$ matrix $\Gamma_{i}$ that pads zero entries to the row whose index does not belong to $\mathcal{B}_{i}$,
    we have:
    \begin{small}
        \begin{align}
        \frac{1}{k}\sum_{i=1}^k\Gamma_i\in \Pi(\hat{P}(\S, \X, A|G=o), \hat{P}(\S, \X, A|G=e)),
        \label{eq: mot_step1}
        \end{align}
    \end{small}
    and
    \begin{small}
        \begin{align}
        &~~~~~OT(\hat{P}(\S, \X, A|G=o), \hat{P}(\S, \X, A|G=e)) \nonumber \\
        &\leq m\mbox{-}OT(\hat{P}(\S, \X, A|G=o), \hat{P}(\S, \X, A|G=e)).
        \label{eq: mot_step2}
        \end{align}
    \end{small}
\end{theorem}

The proof is in Appendix \ref{proof: ot2}. 
Theorem \ref{theorem: ot2} implies that the OT distance of the joint distribution $P(\S, \X, A)$ between observational and experimental data is upper-bounded by its corresponding $m\mbox{-}OT$ problem.
Solving the $m\mbox{-}OT$ problem not only significantly reduces the high optimization cost,
but also allows us to embed the OT distance into deep learning,
where only batch observational units and full experimental units are considered at each iteration.

\textbf{Conclusion.} 
We have overcome two challenges about the conditional OT distance and high computational cost via Theorem \ref{theorem: ot1} and Theorem \ref{theorem: ot2}.
Combining these theorems,
we could learn the weights $\w_o$ by minimizing $m\mbox{-}OT(\hat{P}(\S, \X, A|G=o), \hat{P}(\S, \X, A|G=e))$,
which is the upper bound of the original $OT^{con}_{\x, a}$.
As a result,
our estimated weights $\w_o$ could achieve Proposition \ref{pro: reweight unconf} approximately,
which leads to the identification of HDRC according to Theorem \ref{thm:identification}.

\subsection{Generalization Bound on Long-term HDCR} 
\label{sec: repre}
Although the OT-induced weights $\w_o$ could remove the unobserved confounders and makes HDRC identifiable,
the counterfactual prediction error about long-term HDRC is still large.
The reason is that the regression can only be trained with the factual outcomes
while the counterfactual ones are infinite due to the continuous treatment, 
and the confounding bias exists since $\X \not\Vbar A$.
To address it,
we now derive the generalization bound about the counterfactual prediction error based on the OT-based reweighted distribution,
where the unobserved confounding bias has been removed.

We first define some notations.
Let $\phi:\cX\rightarrow \cZ$ be a representation function with inverse $\psi$, 
where $\cZ$ is the representation space. 
Let $g:\cZ\times \cA\rightarrow \cS$ and $h:\cZ\times \cA \times \cS\rightarrow \cY$ be the hypotheses that predict the potential short- and long-term outcomes.
Since we aims to predict potential long-term outcome $Y(a)$ for all possible treatments $a$ on sample $\x$,
which requires that we could already predict the potential short-term outcomes $\S(a)$ accurately.
As a result,
we define the combined loss for learning these two hypotheses $g, h$ as $\ell_{\phi, g, h}(\x, a)=\ell^s_{\phi, g}(\x, a)+\ell^y_{\phi, h}(\x, a)$,
where 
$\ell_{\phi, g}^{s}(\x,a) =\int _{\cS} L( \s(a) ,g( \phi ( \x) ,a)) P(\s( a) |\x) d\s( a)$
and 
$\ell_{\phi, h}^{y}(\x,a) =\int _{\cS \times \cY} L( y(a) ,h( \phi ( \x) ,a,\s( a))) P( y( a) ,\s( a) |\x) d\s( a) dy( a)$
($L(\cdot, \cdot)$ is the error function).
The ideal prediction error over all treatment $a$ for $\x$ is defined as $\cE(\x)=\bE_{a\sim P(A)}[\ell_{\phi, g, h}(\x, a)]$,
and the target of counterfactual prediction is minimizing $\cE_{cf}=\bE_{\x\sim P(\X)}[\cE(\x)]$.

However, 
$\cE_{cf}$ is practically incomputable,
because the factual data we access lack the counterfactual outcomes.
As a result,
we can only obtain the prediction error on the reweighted factual distribution as $\cE_f^{\w_o}=\bE_{\x, a\sim P^{\w_o}(\X, A)}[\ell_{\phi, g, h}(\x, a)]$.
Next,
following \cite{shalit2017estimating,bellot2022generalization},
we derive a generalization bound to bridge a connection between $\cE_{cf}$ and $\cE_f^{\w_o}$ in the long-term HRDC estimation scenario:
\begin{theorem}
\label{theorem: bound}
    Assuming a family $\cM$ of function $m: \cZ\times \cA\rightarrow \mathbb{R}$, 
    and there exists a constant $B_{\phi}>0$ such that $\frac{1}{B_{\Phi}}\cdot \ell_{\phi, g, h}(\x, a)\in \cM$,
    then we have:
    \begin{small}
            \begin{align}
        \cE_{cf} \leq \cE_f^{\w_o} + IPM_\cM(P(\Z)P(A),~ \w_o\cdot P(\Z,A)),
    \end{align}
    \end{small}
    where $IPM_{\cM}(p,q)=\sup_{m\in \cM}\left|\int_\mathcal{X}g(x)(p(x)-q(x))dx\right|$ is integral probability metric for a chosen $\cM$, 
    $P(\Z)$ and $P(\Z, A)$ are distributions induced by the map $\phi$ from $P(\X)$ and $ P(\X, A)$.
\end{theorem}

The proof is given in Appendix \ref{proof: bound}.
Theorem \ref{theorem: bound} states that the counterfactual error $\cE_{cf}$ is upper-bounded by the factual error $\cE_f^{\w_o}$ plus a reweighted IPM term that measures the dependence between treatment $A$ and observed confounders $\X$.
The theorem inspires a classical balanced representation method \cite{shalit2017estimating},
which handles the observed confounding bias by learning the balanced representation $\Z=\phi(\X)$ that is independent of treatment $A$, i.e., $\Z\Vbar A$.

\section{Model Architecture}
\label{sec: model}
We have presented the theoretical results,
which handle the unobserved and observed confounders by reweighting and representation learning, 
respectively.
In this section, 
we summarize the model architecture inspired by the theoretical results. 
As shown in Fig. \ref{fig: model}, 
LEARN can be divided into three modules: 
OT weight module,
balanced representation module
and long-term estimation module.

\textbf{OT Weight Module.} 
Based on Theorems \ref{theorem: ot1} and \ref{theorem: ot2}, this module estimates weights $\hat{\w}_o$ to remove unobserved confounding bias.
To embed the $m\mbox{-}OT(\alpha, \beta)$ problem into deep learning,
following \cite{yang2023prototypical},
we approximate it by solving $OT(\alpha_{\mathcal{B}}, \beta)$ over each mini-batch $\mathcal{B}$.
At each iteration, 
batch weights $\hat{\w}_{o, \mathcal{B}}=\{\hat{w}_{o, 1}, ..., \hat{w}_{o, b}\}$ of observational data are learned by minimizing $OT(\hat{P}_{\mathcal{B}}(\S, \X, A|G=o), \hat{P}(\S, \X, A|G=e))$ between the batch observational data and full experimental data.
However,
optimizing the OT problem directly usually induces a sparse solution,
leading to low data efficiency
\cite{blondel2018smooth,vincent2022semi}.
Motivated by \cite{yan2024exploiting},
we apply a negative entropy regularization on the marginal distributions $\hat{\w}_{o, \mathcal{B}}$,
i.e.,
$\Omega(\hat{\w}_{o, \mathcal{B}})=\sum_{j=1}^{b}\hat{\w}_{o, j}(\log \hat{\w}_{o, j} - 1)$,
to encourage more observational units to be transported,
and avoid expensive linear programming \cite{2013sinkhornCuturi}.
Finally,
the weights $\hat{\w}_{o, \mathcal{B}}$ are learned by:
\begin{small}
    \begin{align}
    \min_{\hat{\w}_{o, \mathcal{B}}} \min_{\gamma \in \Pi(\alpha_{\mathcal{B}}, \beta)}
    &\langle \gamma, C \rangle
    + \lambda_e \Omega(\gamma),\nonumber\\
    s.t.~
    \Pi(\alpha_{\mathcal{B}}, \beta)&=\{\gamma\in \bR^{b\times n_e}|\gamma \mathbf{1}_{n_e}=\hat{\w}_{o, \mathcal{B}},\;\gamma^T \mathbf{1}_b=\hat{\bm{\mu}}, \;\gamma_{ij}\in \left[0, 1\right]\}. \nonumber
\end{align}
\end{small}

Here $\lambda_e$ is a hyperparameter, the cost matrix $\C$ is constructed as
$C_{ij}=\beta_r d(\bar{\r}_i,\bar{\r}_j)+\beta_x d(\x_i,\x_j)+\beta_a d(a_i,a_j)$,
where $\beta_r,\beta_x,\beta_a$ are hyperparameters, $d(\cdot,\cdot)$ is the Euclidean distance, $i,j$ denote the indices of two units.
The distance of short-term outcomes $\s$ is computed from the GRU mean embedding in the long-term estimation module (details as below) as $\bar{\r}=\frac{1}{t_0}\sum_{t=1}^{t_0}\r_t$.
Therefore,
we need pre-train the GRU firstly to ensure that its embedding is meaningful.

Following \cite{yan2024exploiting},
we develop a projected mirror descent \cite{nemirovskij1983problem,raskutti2015information} to solve the above problem,
which firstly performs proximal gradient descent with the Bregman divergence \cite{banerjee2005clustering},
and then obtains a feasible solution in the set $\Pi(\alpha_{\mathcal{B}}, \beta)$ by projection.
The details can be found in Appendix \ref{app: optim}.

\textbf{Balanced Representation Module.} 
Based on Theorem \ref{theorem: bound},
we design this module to correct the observed confounding bias. 
We transform the observed confounders $\X$ into balanced representation $\Z$ through MLP $\phi(\cdot)$ via minimizing $IPM(P(\Z)P(A), \w_{o, \mathcal{B}} P(\Z, A))$.
In practice,
we use the Wasserstein distance \cite{villani2009optimal} as the implementation of the IPM.
To calculate it, 
we simulate samples standing for the product of marginal distributions $P(\Z)P(A)$ by randomly permuting the observed treatment $A$, 
and the original samples are drawn from the joint distribution $P(\Z, A)$.

\begin{figure}[!t]
\vspace{-0.3cm}
    \centering
    \includegraphics[width=.98\linewidth]{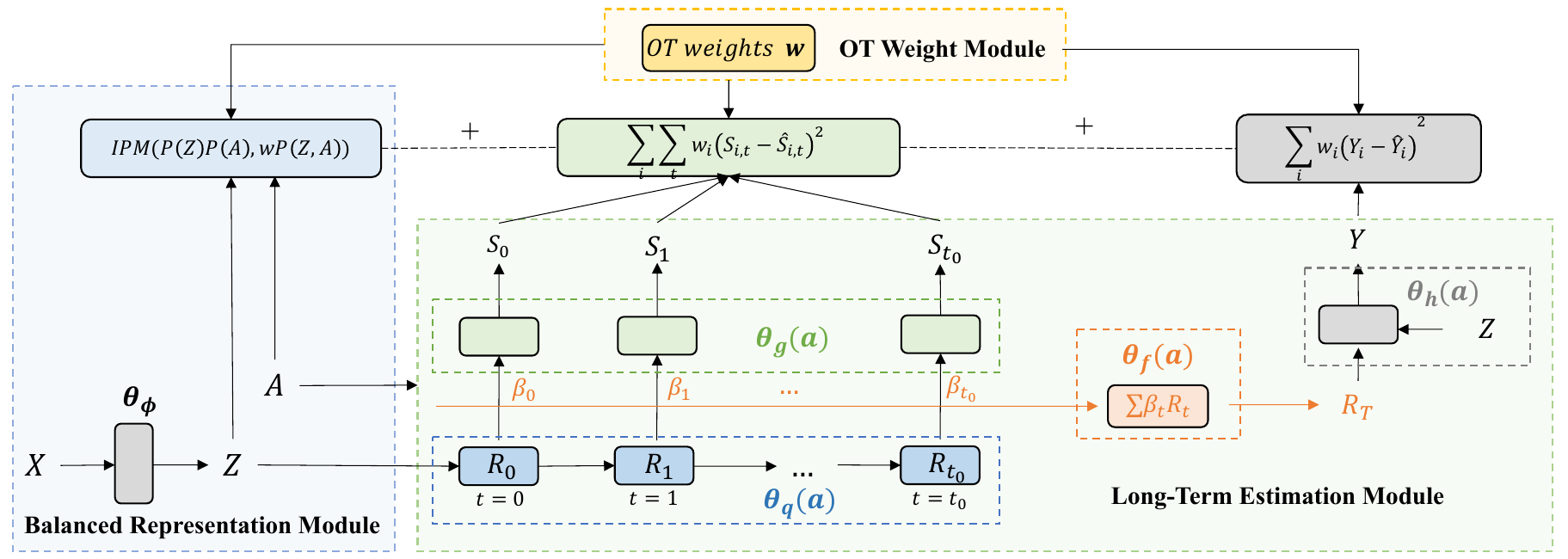}
    \caption{Model architecture of the proposed LEARN.}
    \label{fig: model}
\vspace{-0.35cm}
\end{figure}
\textbf{Long-Term Estimation Module.} This module aims to estimate the potential long-term outcome $Y(a)$ based on potential short-term outcomes $\S(a)$ for different units.
We follow LTEE \cite{cheng2021long} to model the relationship between the short- and long-term outcomes.
Specifically,
the balanced representations $\Z$ are firstly fed into the GRU $q(\cdot)$ to obtain the short-term representations $\R_t$ at each timestep $t$,
which is then used to predict the corresponding short-term outcome $\hat{S}_t$ with a shared MLP $g(\cdot)$.
After that,
we leverage the attention mechanism $f(\cdot)$ \cite{bahdanau2014neural} to construct the long-term representation $\R_T$ from $\{\R_1, \R_2, ..., \R_{t_0}\}$ as 
\begin{small}
    \begin{align}
    \tilde{\R}_t=\tanh{(\W\R_t+\b)},~~~
    \alpha_t=\frac{\exp{\tilde{\R}_t^\top \V}}{\sum_{t=1}^{t_0} \exp{\tilde{\R}_t^\top \V}},~~~
    \R_T=\sum_{t=1}^{t_0}\alpha_t \R_t,\nonumber
\end{align}
\end{small}
where $\W, \b, \V$ are the learnable parameters.
Then base on $\R_T$,
the long-term outcome $\hat{Y}_T$ is through an MLP $h(\cdot)$.

However, different from the original LTEE which only works with binary treatment,
we employ the varying coefficient structure \cite{nie2021vcnet} to extend it to continuous treatment.
In particular, 
we use splines to model the parameters of the above four network structures $q(\cdot), g(\cdot), f(\cdot), h(\cdot)$. 
As a result,
the influence of the treatment is not lost in high-dimensional representations, which in practice has been shown to lead to better performance.

\textbf{Loss Function.} The total training loss at each iteration is:
\begin{small}
    \begin{align}
    \mathcal{L}_\theta
    = &\frac{1}{b}\sum_{i=1}^{b}\hat{w}_{o, i}(y_i - \hat{y}_i)^2
    + \frac{\lambda_o}{b} \sum_{i=1}^{b}\sum_{t=1}^{t_0}\hat{w}_{o, i}(s_{i, t}-\hat{s}_{i, t})^2 \nonumber\\ 
    &+ \frac{1-\lambda_o}{n_e}\sum_{i=n_o+1}^{n_o+n_e}\sum_{t=1}^{t_0}(s_{i, t}-\hat{s}_{i, t})^2\nonumber\\ 
    &+ \lambda_b  IPM (\{\phi(\x_i),\tilde a_i\}_{i=0}^b,  \{\hat{w}_{o, i} \phi(\x_i),\hat{w}_{o, i} a_i\}_{i=0}^b ), \nonumber
\end{align}
\end{small}
where $b$ is the batch size of observational data, $\lambda_b$ is the strength of IPM, and $\tilde a_i$ is the samples obtained by randomly permuting original treatment samples $a_i$.
To make full use of the datasets, 
besides the observational data, 
we also train the model based on the prediction error of the short-term outcomes on experimental data,
and use $\lambda_o$ to control its strength.

\begin{table}[t]
\vspace{-0.15cm}
\caption{The MISE of long-term HDRC estimation is reported as $\mbox{Mean}\pm \mbox{Std}$. The best-performing method is bolded.}
\label{tab: main result}
\centering
\resizebox{.48\textwidth}{!}{
\begin{tabular}{@{}cccccccc@{}}
\toprule
\multirow{2}{*}{Dataset} & \multicolumn{5}{c}{Synthetic } & \multirow{2}{*}{News} & \multirow{2}{*}{TCGA} \\ \cmidrule(lr){2-6}
    &  $\beta_U=1$    &  $\beta_U=1.5$    &   $\beta_U=2.0$   &  $\beta_U=2.5$   & $\beta_U=3$    &  &    \\ \midrule
DRNet &  $26.4\pm7.3$ &  $29.0\pm7.5$ & $33.8\pm8.8$ & $39.4\pm9.4$ & $47.9\pm10.5$   &  $1.97\pm0.27$  & $1.08\pm0.51$                     \\
SCIGAN &  $26.5\pm6.9$ &   $30.5\pm7.5$ &  $34.9\pm8.4$ & $40.5\pm8.4$ & $48.1\pm9.7$ &   $2.45\pm0.32$ & $1.42\pm1.21$      \\
VCNet & $24.2\pm7.2$ & $27.9\pm8.0$ & $31.9\pm9.0$ & $36.4\pm9.9$ & $43.1\pm11.2$ &  $1.52\pm0.22$ & $0.51\pm0.61$     \\
ADMIT &   $23.0\pm7.0$ &$27.0\pm6.8$ &$30.8\pm7.9$ &$34.1\pm8.6$ &$41.4\pm10.2$ &   $1.41\pm0.17$ & $0.49\pm0.58$ \\
ACRF  &  $22.9\pm7.1$ &$26.8\pm6.6$ &$30.7\pm7.1$ &$34.3\pm7.2$ &$41.0\pm11.0$ &     $1.54\pm0.27$ & $0.53\pm0.57$ \\ \midrule
LTEE(VCNet)  &  $22.6\pm6.8$ &$25.9\pm6.5$ &$30.0\pm8.0$ &$33.8\pm8.4$ &$40.3\pm9.4$ &      $1.35\pm0.16$ & $0.47\pm0.55$ \\  \midrule
LEARN(None)  &  $23.4\pm7.4$ &$26.4\pm7.1$ &$30.3\pm8.2$ &$34.1\pm9.1$ &$40.7\pm10.3$ &     $1.36\pm0.13$ & $0.47\pm0.56$ \\
LEARN(IPM) & $22.6\pm7.2$ &$25.8\pm6.7$ &$29.8\pm8.1$ &$33.5\pm8.4$ &$39.9\pm9.5$ &      $1.35\pm0.13$ & $0.46\pm0.56$ \\
LEARN &  \pmb{$22.0\pm6.8$} & \pmb{$25.2\pm6.5$} & \pmb{$28.1\pm6.7$} & \pmb{$33.1\pm8.3$} & \pmb{$38.6\pm9.5$}    &          \pmb{$1.31\pm0.16$}       &   \pmb{$0.46\pm0.55$} \\ \bottomrule
\end{tabular}
}
\vspace{-0.3cm}
\end{table}

\section{Experiments} 
In this section, 
we conduct experiments to validate the effectiveness and correctness of the proposed \textbf{LEARN}.
In particular, we aim to answer the following research questions (RQs):

\textbf{RQ1}: How does LEARN perform in long-term HDRC estimation compared to existing methods?

\textbf{RQ2}: Do the designed modules effectively address both observed and unobserved confounding bias?

\textbf{RQ3}: How does batch size affect the approximation error about $m\mbox{-}OT$ problem in Theorem \ref{theorem: ot2}?

\textbf{RQ4}: How does the size of experimental data used for data combination affect the performance?

\textbf{RQ5}: Does our model stably perform well under different choices of hyperparameters?

\subsection{Datasets}
\label{sec: exp}
Since the true long-term HDRC is unavailable for real-world datasets, 
following prior works \cite{nie2021vcnet,wang2022generalization,kazemi2024adversarially},
we use five synthetic datasets with varying levels of unobserved confounding bias controlled by $\beta_U \in \{1, 1.5, 2, 2.5, 3\}$, 
and two semi-synthetic datasets, 
News \cite{schwab2020learning} and TCGA \cite{weinstein2013cancer}, 
to evaluate our model.

\textbf{Synthetic Data Generation.}
Following Fig. \ref{fig: graphs}, 
we first sample $n_o$ observational units as
$X^o\sim\mathcal{N}(\mathbf{0.1}_p,\mathbf{I}_p), U\sim\mathcal{N}(\mathbf{0.25}_q,\mathbf{I}_q)$,
and $n_e$ experimental units as 
$X^e\sim\mathcal{N}(\mathbf{0.2}_p,\mathbf{I}_p), U\sim\mathcal{N}(\mathbf{0.25}_q,\mathbf{I}_q)$.
With $W_j^x\sim\mathcal{U}(0,0.5), W_j^u\sim\mathcal{U}(0.2,0.5)$.
we then generate treatment as $A=\sigma \left(\sum_{j=1}^{p/3} W_j^x\sin{(X_j)}+\sum_{j=p/3}^{2p/3}W_j^x X_j^2+\beta_U\bar{U}\mathbb{I}(G=o)\right)$,
generate short-term outcomes ($t=0,\ldots,7$) as $S_t=\sum_{j=p/3}^{2p/3}W_j^x X_j^2 + (t+5)A(\sum_{j=2p/3}^{p}W_j^x X_j)
+ \beta_UtA\left(\sum_{j=1}^q W_j^u\cos{(U_j)}\right)
+ 0.5\bar{S}_{1:t-1}+\mathcal{N}(0, 0.5)$,
generate long-term outcome $Y=0.25\sum_{j=1}^{p} W_j^x X_j^2 + \frac{A}{7}\sum_{t=1}^{7} e^{1/t} S_t + \mathcal{N}(0, 0.5)$.
We set $n_o=10000$, $n_e=500$, $p=15$, and $q=5$.
We construct 5 datasets with $\beta_U\in\{1,1.5,2,2.5,3\}$ to control the strength of unobserved confounding.

\textbf{Semi-synthetic Data Generation. }
We use News \cite{schwab2020learning} and TCGA \cite{weinstein2013cancer}, containing 5,000 NY Times articles and gene-expression data from 9,659 cancer patients, respectively.
For each dataset, we reuse the original covariates, split samples into observational and experimental subsets at 9:1, and split covariates into observed $X$ and unobserved $U$ at 8:2.
Following \cite{nie2021vcnet,wang2022generalization}, let $\boldsymbol{v}_{i,x}=\boldsymbol{u}_{i,x}/\|\boldsymbol{u}_{i,x}\|$ and $\boldsymbol{v}_{i,u}=\boldsymbol{u}_{i,u}/\|\boldsymbol{u}_{i,u}\|$, where $\boldsymbol{u}_{i,x},\boldsymbol{u}_{i,u}\sim\mathcal{N}(\mathbf{0},\mathbf{1})$ for $i=1,2,3$.
We sample $A\sim\mathrm{Beta}(\gamma,\beta)$ with $\gamma=2$, $\beta=(1)/d^{*}$, and $d^{*}=\left|\frac{\boldsymbol{v}_{3,x}^{T}X}{2\boldsymbol{v}_{2,x}^{T}X}+\frac{\boldsymbol{v}_{3,u}^{T}U}{2\boldsymbol{v}_{2,u}^{T}U}\mathbb{I}(G=o)\right|$.
For $t=0,\ldots,7$, 
we first generate 
$S_t'= \max\{-2,\min[2,\exp(\frac{\boldsymbol{v}_{2,x}^{T}X}{\boldsymbol{v}_{3,x}^{T}X}+\frac{\boldsymbol{v}_{2,u}^{T}U}{\boldsymbol{v}_{3,u}^{T}U}-0.3)]\}+\alpha t(\boldsymbol{v}_{1,x}^{T}X+\boldsymbol{v}_{1,u}^{T}U)$, 
then we generate short- and long-term outcomes as 
$S_t=8(A-0.5)^2\sin(\frac{\pi}{2}A)S_t'+0.5\bar S_{1:t-1}+\mathcal{N}(0,0.5)$, and $Y=\lambda(\boldsymbol{v}_{1,x}^{T}X+\boldsymbol{v}_{2,x}^{T}X)+\frac{A}{7}\sum_{t=1}^{7}e^{1/t}S_t+\mathcal{N}(0,0.5)$.
For News, we select $\alpha=1, \lambda=5$, and for TCGA, we select $\alpha=20, \lambda=50$.

\subsection{Baselines and metrics}

\textbf{Baselines.} 
We first compare our model with several baselines designed for continuous treatment,
including \textbf{DRNet} \cite{schwab2020learning}, 
\textbf{VCNet} \cite{nie2021vcnet},
\textbf{SCIGAN} \cite{bica2020estimating},
\textbf{ADMIT} \cite{wang2022generalization}
and \textbf{ACFR} \cite{kazemi2024adversarially}.
Besides,
since there is a lack of work on estimating long-term HDRC,
we adapt LTEE \cite{cheng2021long} to the continuous treatment setting by incorporating a varying coefficient structure, denoted as \textbf{LTEE(VCNet)}.
Moreover, to evaluate the contribution of each component in our model, we also include two ablated variants:
\textbf{LEARN(None)} removes all modules addressing confounding bias;
\textbf{LEARN(IPM)} includes only the IPM term to correct observed confounding bias, without accounting for the unobserved one.

\begin{table}[!t]
\caption{The relative approximation error of different batch size about $m\mbox{-}OT$ problem in Theorem \ref{theorem: ot2}.}
\label{tab: batch size}
\centering
\resizebox{.27\textwidth}{!}{
\begin{tabular}{@{}ccccc@{}}
\toprule
Batch Size & 512 & 1024 & 2048 & 4096 \\ \midrule
Relative Error  &  $14.51\%$   &  $7.68\%$   &  $3.45\%$   & $1.03\%$     \\ \bottomrule
\end{tabular}
}
\vspace{-0.2cm}
\end{table}

\textbf{Metrics.} 
we first split the observational dataset with a $6/2/2$ ratio into training, validation and test sets,
and we report the mean integrated square error (MISE) and standard deviations over 10 independent replications on the test set:
$\mbox{MISE}=\frac{1}{n}\sum_{i=1}^n \int_{a_1}^{a_2} (\mu(\x_i, a) - \hat{\mu}(\x_i, a))^2 da$,
where $n$ is the test sample size  and $[a_1, a_2]$ is the sampling interval of treatment values.

\textbf{Implementation details. }
All MLP is 2 fully connected layers with 50 hidden units and Elu activation.
For the varying coefficient structure, we use truncated poly-nomial basis with degree 2 and two knots at {1/3, 2/3}.
For the hyperparameters, we tune the strength of the Wasserstein distance $\lambda_b\in \{1, 10, 50, 100\}$,
the strength of negative entropy regularization $\lambda_e\in \{1e-4, 1e-2\}$,
the learning rate about the projected mirror descent as $1e-3$,
and the iterations about the projected mirror descent as $\{50, 100\}$,
the strength of short-term outcome loss as $\lambda_o\in\{0.25, 0.5, 0.75\}$.
Besides,
we pre-train the model for 100 epochs without weighting,
and then train for 400 epochs with early stopping.
We use Adam \cite{kingma2014adam} with learning rate of $1e-3$ and weight decay of $5e-4$.
All the experiments run on the Nvidia RTX A6000.

\begin{table}[!t]
\caption{The impact of experimental data's sample size.}
\label{tab: sample size}
\centering
\resizebox{.4\textwidth}{!}{
\begin{tabular}{@{}cccccc@{}}
\toprule
Number & 100 & 250 & 500 & 1000 & 2000 \\ \midrule
MISE   &  $22.52\pm7.40$   &  $22.18\pm7.19$   &  $21.98\pm6.81$   & $21.63\pm6.83$     & $21.43\pm6.74$     \\ \bottomrule
\end{tabular}
}
\vspace{-0.1cm}
\end{table}

\subsection{Result and Analysis}
\textbf{Overall Performance (RQ1).}
Table \ref{tab: main result} presents the results on synthetic and semi-synthetic datasets.
First,
LEARN achieves the highest accuracy across all datasets, outperforming baselines only designed for continuous treatments. 
This gains mainly come from the long-term estimation module to model the relationship between short- and long-term outcomes,
which also benefit LTEE(VCNet).
Second,
LEARN(IPM) outperforms LEARN(None), indicating that IPM-based balanced representations reduce observed confounding bias, thereby lowering MISE.
Finally,
LEARN further surpasses LEARN(IPM), verifying the benefit of correcting unobserved confounding with the OT weight module.

\textbf{Confounding Bias Correction (RQ2).}
We assess whether the proposed modules reduce both observed and unobserved confounders on the synthetic dataset ($\beta_U=1$).
For observed confounders, we assess the balanced representation module using the Hilbert-Schmidt Independence Criterion (HSIC, \cite{gretton2005measuring}).
The dependence between $\phi(\X)$ and $A$ is reduced by \textbf{83.45\%}.
For unobserved confounders, we evaluate the OT weight module using the Hilbert--Schmidt Conditional Independence Criterion (HSCONIC, \cite{fukumizu2007kernel,sun2007kernel}) to assess the post-reweighting unconfoundedness condition in observational data,
i.e., $A\Vbar \{Y(a),\S(a)\}\mid\X,G=o$.
Reweighting reduces the conditional dependence between $A$ and $Y(a)$, $\S(a)$ given $\X$ by  \textbf{67.71\%} and \textbf{71.55\%}, respectively.

\textbf{Sensitivity of Batch Size (RQ3).}
Following \cite{sommerfeld2019optimal}, we analyze the sensitivity of batch size in the $m$-OT problem by computing the relative approximation error:
$\frac{|m\mbox{-}OT_B(\alpha, \beta) -  OT^*(\alpha, \beta)|}{OT^*(\alpha, \beta)}$,
where $OT^*(\alpha, \beta)=\langle \Gamma^*, C \rangle$ is the optimal transport cost optimized using full observational data,
and $m\mbox{-}OT_B(\alpha, \beta)$ is the approximate cost with batch size $B$.
Table \ref{tab: batch size} reports the results on synthetic dataset ($\beta_U=1$),
which shows that smaller batches lead to higher relative approximation error, reflecting a trade-off between computational efficiency and accuracy about $m\mbox{-}OT$ problem.
However, even with the smallest batch size of 512, the relative error 14.51\% is acceptable.

\textbf{Impact of Sample Size on experimental data (RQ4).}
Since LEARN relies on data combination, 
we evaluate the impact of the experimental sample size using the synthetic dataset ($\beta_U = 1$). 
Table \ref{tab: sample size} shows that MISE decreases with larger sample sizes, 
as a larger experimental data provides a more reliable distribution for aligning the observational data, 
which allows us to learn more effective weights for eliminating unobserved confounding bias.

\textbf{Hyperparameter Sensitivity Study (RQ5).} 
We study the sensitivity to $\lambda_b, \lambda_o, \lambda_e$ in Fig. \ref{fig:sensitivity-study}.
Firstly,
we observe that MISE increases when $\lambda_b$ becomes small or large, 
showing a trade-off between representation balancing and outcome estimation.
Secondly,
the results of varying $\lambda_o$ shows that jointly using short-term outcomes from observational and experimental data improves performance.
Finally,
the negative entropy regularization in OT is beneficial, 
but large $\lambda_e$ will push the learned weights close to the uniform distribution, resulting in a failure of alignment.

\section{Conclusion}
In this paper, 
we provide a practical solution to estimate the long-term HDRC with unobserved confounders and continuous treatment via data combination,
which is not well-studied in existing work.
Specifically,
to remove unobserved confounders and make long-term HDRC identifiable, 
we first propose a novel weighting schema aligning the conditional distribution of short-term outcomes between observational and experimental groups.
We solve the reweighting problem under an optimal transport framework, 
showing that the conditional distribution discrepancy can be bounded by mini-batch joint distribution discrepancy, 
which is computationally efficient.
Further, to handle the observed confounders,
we derive a generalization bound on the counterfactual outcome regression error
based on the OT-induced reweighted distribution.
Building upon the above theoretical results,
we develop our model LEARN to estimate long-term HDRC accurately.
Extensive experimental results verify the correctness of our theory and the effectiveness of LEARN.

\begin{figure}[!t]
\vspace{-0.1cm}
\centering
\begin{subfigure}[b]{0.31\linewidth}
    \centering
    \includegraphics[width=\linewidth]{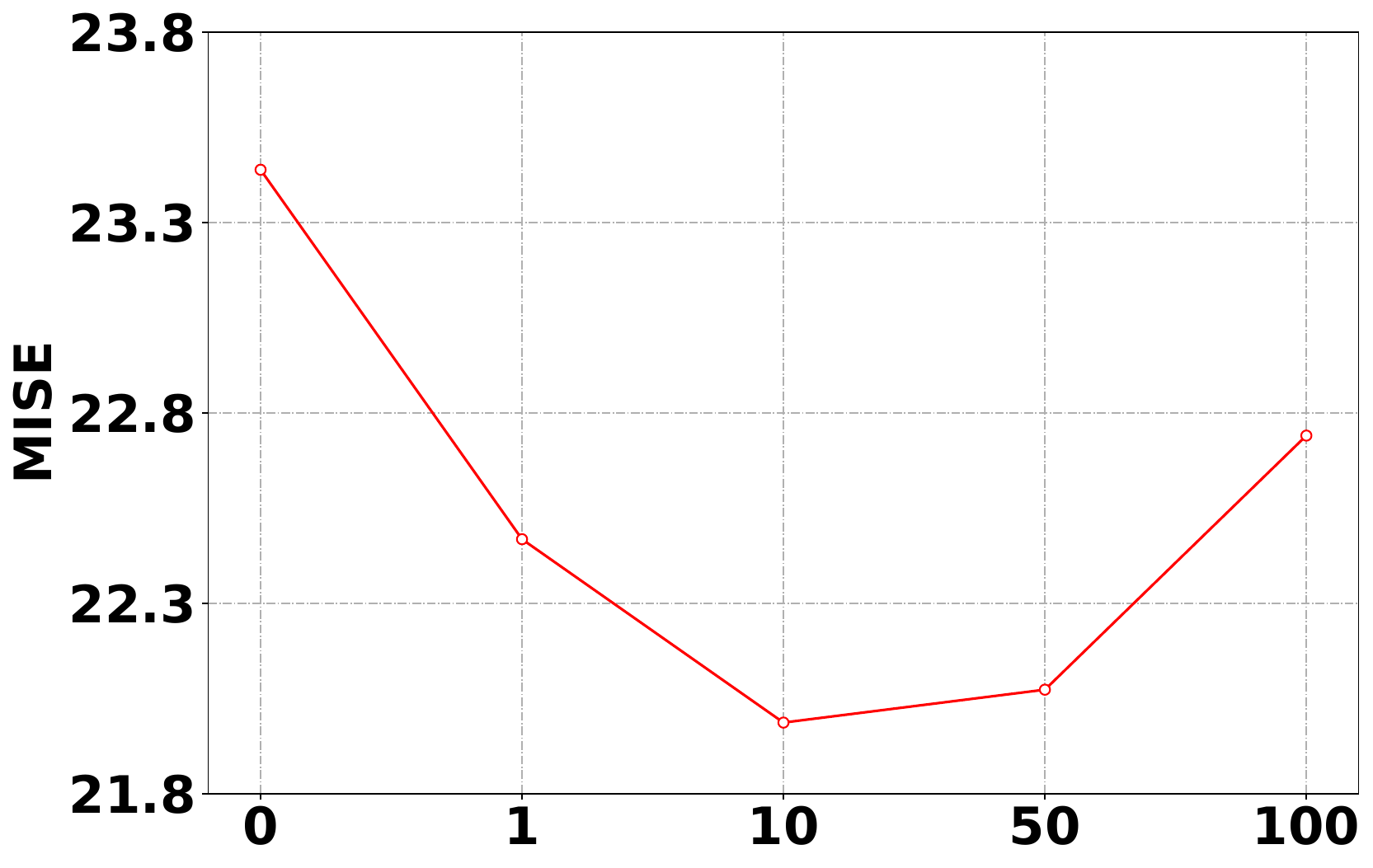}
    \caption{$\lambda_b$}
    \label{fig:lambda_b}
\end{subfigure}
\hfill
\begin{subfigure}[b]{0.31\linewidth}
    \centering
    \includegraphics[width=\linewidth]{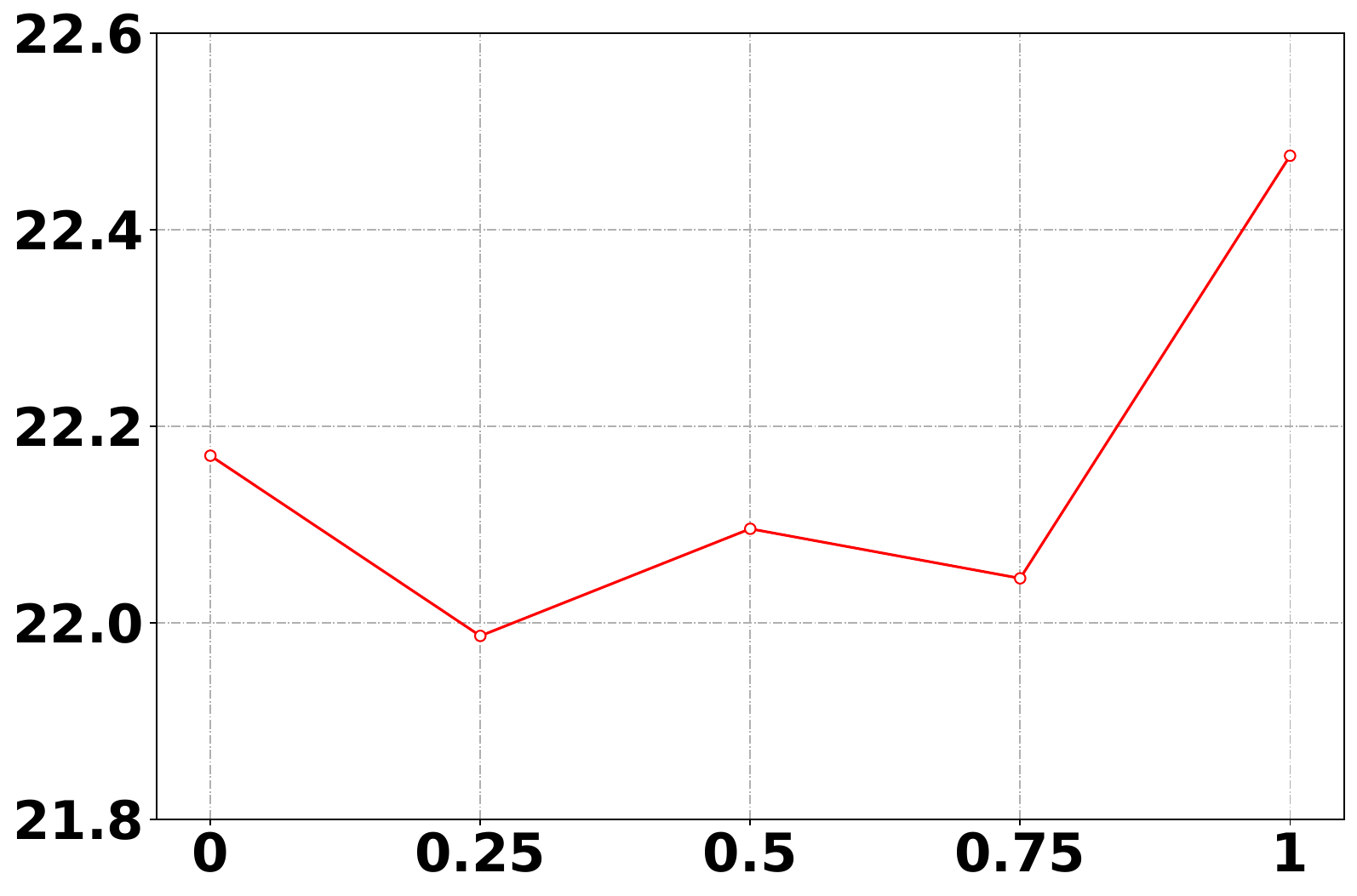}
    \caption{$\lambda_o$}
    \label{fig:lambda_o}
\end{subfigure}
\hfill
\begin{subfigure}[b]{0.31\linewidth}
    \centering
    \includegraphics[width=\linewidth]{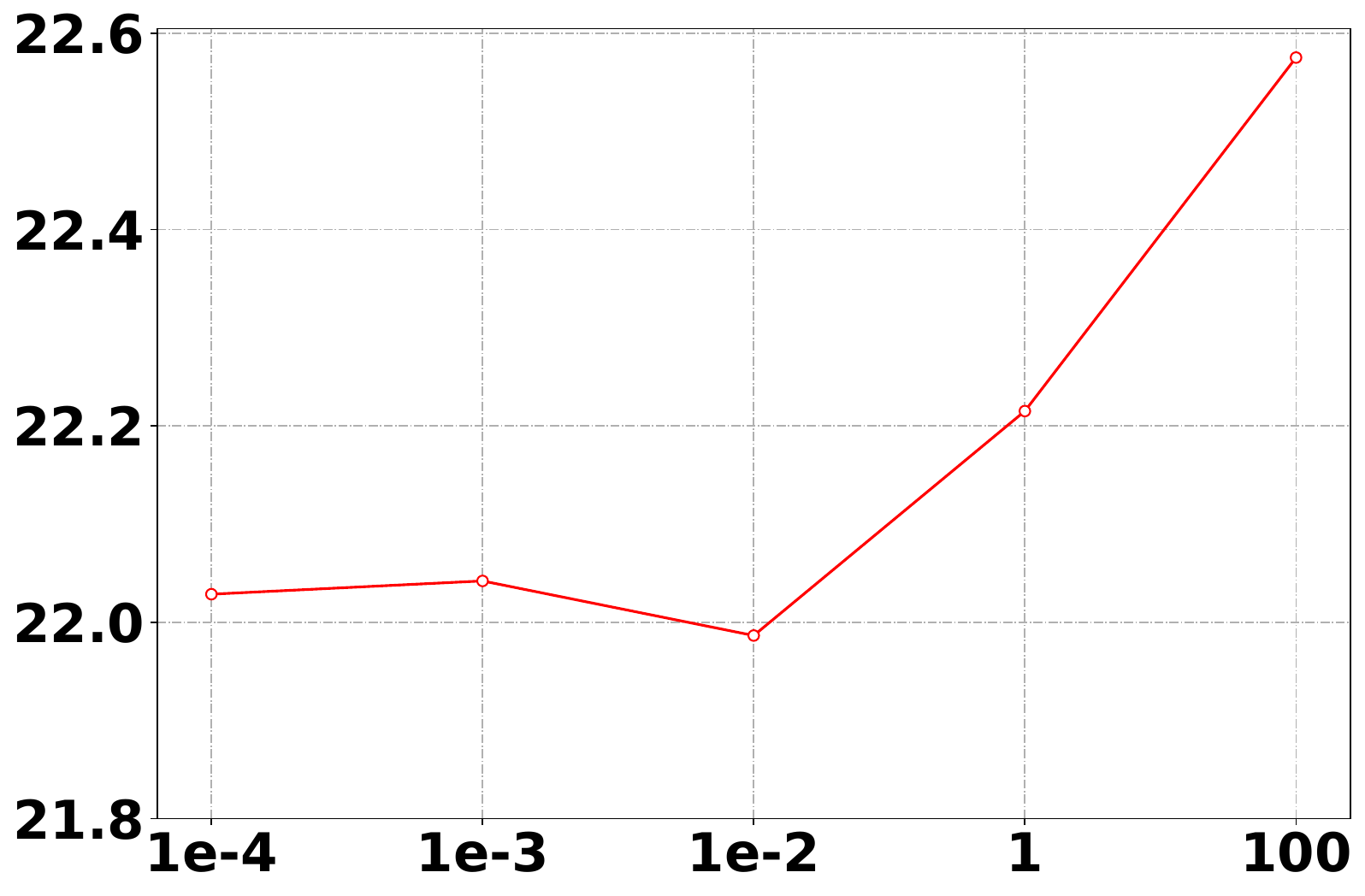}
    \caption{$\lambda_e$}
    \label{fig:lambda_e}
\end{subfigure}

\caption{Impact of the hyperparameters: (a) IPM $\lambda_b$, 
(b) short-term loss $\lambda_o$, and 
(c) negative entropy regularization $\lambda_e$.}
\label{fig:sensitivity-study}
\vspace{-0.1cm}
\end{figure}

\begin{acks}
This research was supported in part by National Science and Technology
Major Project (2021ZD0111501), Natural Science Foundation of China
(U24A20233), and CCF-DiDi GAIA Collaborative Research Funds
(CCF-DiDi GAIA 202403).
\end{acks}

\appendix

\section{Proof Details}
\numberwithin{equation}{section}
\setcounter{equation}{0}
\label{sec: proof}

\subsection{Proof of Proposition \ref{pro: reweight unconf}}

\label{proof: reweight unconf}


\begin{proof}
    We will show that we have:
    \begin{align}
        P^\w\big(\S(a)\mid \X=\x, A=a, G=o\big) 
        &= P\big(\S(a)\mid \X=\x\big),
        \label{eq:prop1-step1}
        \\
        P^\w\big(\S(a)\mid \X=\x, G=o\big) 
        &= P\big(\S(a)\mid \X=\x\big).
        \label{eq:prop1-step2}
    \end{align}
    Then, combining \eqref{eq:prop1-step1}--\eqref{eq:prop1-step2} finishes the proof.
    
    We first prove \eqref{eq:prop1-step1}.
    By Assumption~\ref{assum: consist}, on the event $\{A=a\}$ we have
    $\S = \S(a)$, so the reweighting condition in Eq. \eqref{eq:SwGXA} implies
    \begin{align}
        P^\w\big(\S(a)\mid \X, A=a, G=o\big)
        =
        P^\w\big(\S(a)\mid \X, A=a, G=e\big),
        \label{eq:Sa-GoGe}
    \end{align}
    i.e., conditioning on $(\X,A=a)$, the distribution of $\S(a)$ is the same for $G=o$ and $G=e$ under $P^\w$.
    
    Since the experimental units $G=e$ receive constant weights $\bm\mu$, the restriction of $P^\w$ to $\{G=e\}$ coincides with the original distribution $P$ on $\{G=e\}$. Hence,
    \begin{align}
        P^\w\big(\S(a)\mid \X, A=a, G=e\big)
        = P\big(\S(a)\mid \X, A=a, G=e\big).
        \label{eq:Pw=P-Ge}
    \end{align}
    By Assumption~\ref{assum: internal validity of exp} ($A \Vbar \S(a) \mid \X, G=e$), we have
    \begin{align}
        P\big(\S(a)\mid \X, A=a, G=e\big)
        = P\big(\S(a)\mid \X, G=e\big).
        \label{eq:exp-unconf}
    \end{align}
    By Assumption~\ref{assum: external validity of exp} ($G \Vbar \S(a) \mid \X$), we have
    \begin{align}
        P\big(\S(a)\mid \X, G=e\big) = P\big(\S(a)\mid \X\big).
        \label{eq:ext-valid}
    \end{align}
    Combining \eqref{eq:Sa-GoGe}--\eqref{eq:ext-valid}, we obtain $P^\w\big(\S(a)\mid \X, A=a, G=o\big)
        = P\big(\S(a)\mid \X\big)$, 
    which proves \eqref{eq:prop1-step1}.
    
    We next prove \eqref{eq:prop1-step2}. Using the law of total probability under $P^\w$, we have $P^\w\big(\S(a)\mid \X, G=o\big)= \int P^\w\big(\S(a)\mid \X, A=a, G=o\big)
            dP^\w(A=a\mid \X, G=o)$.
    Substituting \eqref{eq:prop1-step1} into the integrand yields
    $
        P^\w\big(\S(a)\mid \X, G=o\big)
            = \int P\big(\S(a)\mid \X\big)\,
                dP^\w(A=a\mid \X, G=o)
            = P\big(\S(a)\mid \X\big)
    $, which 
    establises \eqref{eq:prop1-step2}.
        
\end{proof}

\subsection{Proof of Theorem \ref{thm:identification}}
\label{proof: iden}

\begin{proof}
    We first establish the weighted unconfoundedness \eqref{eq:Ya-unconf-weighted}.
    Under $P^\w$, using the law of total probability over $\S(a)$, we have
    \begin{small}
        \begin{align}
        &\;\;\;\;\;\;P^\w\big(Y(a)\mid \X, A=a, G=o\big) \nonumber\\
        &= \int P^\w\big(Y(a),\S(a)\mid \X, A=a, G=o\big)\, d\S(a) \nonumber\\
        &= \int P^\w\big(Y(a)\mid \S(a),\X, A=a, G=o\big)\,
                P^\w\big(\S(a)\mid \X, A=a, G=o\big)\, d\S(a).
        \label{eq:proof-thm-step1}
    \end{align}
    \end{small}
    
    Note that Assumption~\ref{assum: LU} is $A \Vbar Y(a) \mid \X,\S(a),G=o$ under $P$.
    Since the weights $\w_o$ for observational data depend only on $(\X,A,\S)$ and not on $Y$ or $Y(a)$, this conditional independence is preserved under $P^\w$ for $G=o$, i.e.,
    \begin{small}
        \begin{align}
        P^\w\big(Y(a)\mid \S(a),\X, A=a, G=o\big)
        = P^\w\big(Y(a)\mid \S(a),\X, G=o\big).
        \label{eq:Ya-indep-A-weighted}
        \end{align}    
    \end{small}
    Moreover, Proposition~\ref{pro: reweight unconf} implies
    \begin{align}
        P^\w\big(\S(a)\mid \X, A=a, G=o\big)
        = P^\w\big(\S(a)\mid \X, G=o\big).
        \label{eq:Sa-indep-A-weighted}
    \end{align}
    Substituting \eqref{eq:Ya-indep-A-weighted} and \eqref{eq:Sa-indep-A-weighted} into \eqref{eq:proof-thm-step1} yields
    \begin{small}
         \begin{align}
            &\;\;\;\;\;\;\;P^\w\big(Y(a)\mid \X, A=a, G=o\big) \nonumber\\
            &= \int P^\w\big(Y(a)\mid \S(a),\X, G=o\big)\,
                    P^\w\big(\S(a)\mid \X, G=o\big)\, d\S(a) \nonumber\\
            &= P^\w\big(Y(a)\mid \X, G=o\big),
        \end{align}       
    \end{small}
    where the last equality follows from the law of total expectation with respect to $\S(a)$.
    Taking expectations on both sides gives \eqref{eq:Ya-unconf-weighted}.

    We now prove the identification formula \eqref{eq:hdrci-formula}.
    In the original (unweighted) distribution $P$, Assumption~\ref{assum: external validity of exp} implies $G \Vbar Y(a) \mid \X$,
    and thus $\mu(\x,a)
        = \mathbb E\big[Y(a)\mid \X=\x\big]
        = \mathbb E\big[Y(a)\mid \X=\x, G=o\big]$.
    Applying the law of total expectation over $\S(a)$, we obtain
    \begin{small}
         \begin{align}
        \mu(\x,a)
        &= \int \mathbb E\big[Y(a)\mid \X=\x, \S(a)=\s, G=o\big]\,
            dP\big(\S(a)=\s\mid \X=\x, G=o\big).
        \label{eq:mu-expand-Sa}
     \end{align}   
    \end{small}

    By Assumption~\ref{assum: LU} ($ A \Vbar Y(a) \mid \X,\S(a),G=o$),
    we have
    \begin{small}
            \begin{align}
        \mathbb E\big[Y(a)\mid \X=\x, \S(a)=\s, G=o\big]
        = \mathbb E\big[Y(a)\mid \X=\x, \S(a)=\s, A=a, G=o\big].
        \label{eq:Ya-indep-A-P}
    \end{align}
    \end{small}

    By Assumption~\ref{assum: consist}, on the event $\{A=a,G=o\}$ we have $Y=Y(a)$ and $\S=\S(a)$, so
    \begin{small}
          \begin{align}
        \mathbb E\big[Y(a)\mid \X=\x, \S(a)=\s, A=a, G=o\big]
        = \mathbb E\big[Y\mid \X=\x, \S=\s, A=a, G=o\big].
        \label{eq:consistency-Y}
    \end{align}  
    \end{small}

    Substituting \eqref{eq:Ya-indep-A-P} and \eqref{eq:consistency-Y} into \eqref{eq:mu-expand-Sa} yields
    \begin{align}
        &\mu(\x,a) \nonumber
        \\= &\int \mathbb E\big[Y\mid \X=\x, \S=\s, A=a, G=o\big]\,
            dP\big(\S(a)=\s\mid \X=\x, G=o\big).
        \label{eq:mu-with-P-Sa}
    \end{align}
    Proposition~\ref{pro: reweight unconf} further implies that $P\big(\S(a)\mid \X=\x, G=o\big)
        = P^\w\big(\S(a)\mid \X=\x, A=a, G=o\big)$.
    Using consistency $\S(a)=\S$ on $\{A=a,G=o\}$, we obtain
    \begin{align}
        P\big(\S(a)\mid \X=\x, G=o\big)
        = P^\w\big(\S\mid \X=\x, A=a, G=o\big).
        \label{eq:PtoPw-S}
    \end{align}
    Substituting \eqref{eq:PtoPw-S} into \eqref{eq:mu-with-P-Sa}, we have
    \begin{small}
            \begin{align}
        &\mu(\x,a) \nonumber
        \\ =& \int \mathbb E\big[Y\mid \X=\x, \S=\s, A=a, G=o\big]\,
            dP^\w\big(\S=\s\mid \X=\x, A=a, G=o\big).
        \label{eq:mu-with-Pw-S}
    \end{align}
    \end{small}

    Finally, note that the weights $\w_o$ depend only on $(\X,A,\S)$ and do not involve $Y$, so the conditional distribution of $Y$ given $(\X,\S,A,G)$ is unchanged under reweighting:
    \begin{align}
        \mathbb E_{P^\w}\big[Y\mid \X,\S,A,G=o\big]
        = \mathbb E\big[Y\mid \X,\S,A,G=o\big].
        \label{eq:condY-unchanged}
    \end{align}
    Therefore, the right-hand side of \eqref{eq:mu-with-Pw-S} coincides with the conditional expectation of $Y$ under $P^\w$:
    \begin{small}
    \begin{align}
        &\;\;\;\;\;\;\mathbb E_{P^\w}\big[Y\mid \X=\x, A=a, G=o\big] \nonumber\\
        &= \int \mathbb E_{P^\w}\big[Y\mid \X=\x, \S=\s, A=a, G=o\big]\,
            dP^\w\big(\S=\s\mid \X=\x, A=a, G=o\big) \nonumber\\
        &\stackrel{\eqref{eq:condY-unchanged}}{=}
           \int \mathbb E\big[Y\mid \X=\x, \S=\s, A=a, G=o\big]\,
            dP^\w\big(\S=\s\mid \X=\x, A=a, G=o\big).
    \end{align}
    \end{small}

    Compared with \eqref{eq:mu-with-Pw-S},
    \begin{align}
        \mu(\x,a) = \mathbb E\big[Y(a) \mid \X=\x\big] = \mathbb E_{P^\w}\big[Y\mid \X=\x, A=a, G=o\big],
    \end{align}
    then we have \eqref{eq:hdrci-formula}, which means $\mu(\x,a)$ can be identified by reweighted observational data .
\end{proof}

\subsection{Proof of Theorem \ref{theorem: ot1}}
\label{proof: ot1}



We firstly define some notations.
Let $C_{\S,\X,A, \tilde{\S} ,\tilde{\X} ,\tilde{A}}$ be the cost matrix, and define its element as  $C_{\S,\X,A, \tilde{\S} ,\tilde{\X} ,\tilde{A}}(\s, \x, a, \tilde{ \s}, \tilde{\x}, \tilde{a})$, i.e., the cost between the observational sample $\s, \x, a$ and the experimental sample $\tilde{ \s}, \tilde{\x}, \tilde{a}$, and similarly for $C_{\S, \tilde{\S} }$ ,$C_{\X,\tilde{\X}}$ , and $C_{A, \tilde{A}}$, and their elements. We formally state the full Theorem \ref{theorem: ot1} as follows.

\begin{rep_theorem}{theorem: ot1}
    Assuming the cost matrix of the joint distribution $P(\S, \X, A)$ is separable, 
    i.e.,
    $C_{\S, \X, A, \tilde{\S}, \tilde{\X}, \tilde{A}}(\s, \x, a, \tilde{\s}, \tilde{\x}, \tilde{a}) = C_{\S, \tilde{\S}}(\s, \tilde{\s}) + C_{\X, \tilde{\X}}(\x, \tilde{\x}) + C_{A, \tilde{A}}(a, \tilde{a})$,
    and suppose that
    (i) for each pair $ (\x, a) \neq (\tilde \x,\tilde a),  C_{\X, \tilde{\X}}(\x,\tilde \x)+C_{A, \tilde{A}}(a,\tilde a) \geq \delta_{\X, A, \tilde{\X}, \tilde{A}}{(\x,a,\tilde \x,\tilde a)}$,
    (ii) for all pair $\s \text{ and } \tilde \s $, $ 0\leq C_{\S, \tilde{\S}}(\s,\tilde \s)\le M$,
   and (iii) $\delta_{\X, A, \tilde{\X}, \tilde{A}}{(\x,a,\tilde \x,\tilde a)} >  M$, 
    and then we have:
    \begin{small}
        \begin{align}
        OT^{con}_{\x, a}\leq OT(\hat{P}(\S, \X, A|G=o), \hat{P}(\S, \X, A|G=e)).
        \end{align}
    \end{small}
\end{rep_theorem}

Note that the separable cost matrix condition is mild,  widely used to simplify the optimal transport problem \cite{courty2017joint,kim2022conditional}, and many distances, such as the Manhattan distance and the squared Euclidean distance, are separable.
And the condition that $\delta_{\X, A, \tilde{\X}, \tilde{A}}{(\x,a,\tilde \x,\tilde a)} >  M$ is also mild.
These conditions can be directly satisfied by a well-designed cost function. 
Now we prove Theorem 2.

\begin{proof}
    With the separable assumption of the cost matrix,
    we abbreviate the cost matrix about the joint distribution $\hat{P}(\S, \X, A)$ as
    $C_{\S,\X,A, \tilde{\S} ,\tilde{X} ,\tilde{A}}=C_{\S, \tilde{\S}} +C_{\X, \tilde{\X}} +C_{A, \tilde{A}}$.
    Then for any $\Gamma _{\S,\X,A, \tilde{\S} ,\tilde{\X} ,\tilde{A}} \in \Pi(\hat{P}(\S, \X, A|G=o), \hat{P}(\tilde{S} ,\tilde{X} ,\tilde{A}|G=e))$,
    we have:

\begin{small}
     \begin{align*}
&OT(\hat P(\S,\X,A|G=o), \hat P(\tilde{\S} ,\tilde{\X} ,\tilde{A}|G=e)) \\ 
        &=\min_{\Gamma}\langle \Gamma_{\S,\X,A, \tilde{\S} ,\tilde{\X} ,\tilde{A}}  ,C_{\S,\X,A, \tilde{\S} ,\tilde{X} ,\tilde{A}} \rangle \\ 
        & =\min_{\Gamma}\langle \Gamma_{\S,\X,A, \tilde{S} ,\tilde{\X} ,\tilde{A}} ,C_{\S, \tilde{\S}} +C_{\X, \tilde{\X}} +C_{A, \tilde{A}} \rangle \\ 
        &\geqslant \min_{\Gamma}\langle \Gamma _{\S,\X,A, \tilde{\S} ,\tilde{\X} ,\tilde{A}} ,C_{\S, \tilde{\S}} \rangle \\ 
        &=\min_{\Gamma}\sum_{a,\tilde{a}}\sum_{\x,\tilde{\x}}\sum_{\s,\tilde{\s}} \Gamma _{\S,\X,A, \tilde{\S} ,\tilde{\X} ,\tilde{A}}\left( \s,\x,a,\tilde{\s} ,\tilde{\x} ,\tilde{a}\right) C_{\S, \tilde{\S}}(\s, \tilde{\s})  \\ 
        & \geqslant \min_{\Gamma} \sum_{a,\tilde{a}}\sum_{\x,\tilde{\x}}\sum_{\s,\tilde{\s}} \Gamma _{\S,\X,A, \tilde{\S} ,\tilde{\X} ,\tilde{A}}\left( \s, \x,a,\tilde{\s} ,\tilde{\x} ,\tilde{a}\right) C_{\S, \tilde{\S}}(\s, \tilde{\s})  I\left( x=\tilde{x} ,a=\tilde{a}\right)  \\ 
        & =\min_{\Gamma}\sum_{a,\tilde{a}}\sum_{\x,\tilde{\x}}\sum_{\s,\tilde{\s}} \Gamma_{\S,\tilde{\S} |\X,A ,\tilde{\X} ,\tilde{A}}\left( \s,\tilde{\s} |\x,a,\tilde{\x} ,\tilde{a}\right) \Gamma_{\X,A ,\tilde{\X} ,\tilde{A}}\left( \x,a ,\tilde{\x} ,\tilde{a}\right)  C_{\S, \tilde{\S}}(\s, \tilde{\s})  \cdot \\
        &\;\;\;\;\;\;I\left( x=\tilde{x} ,a=\tilde{a}\right) \\ 
        &=\min_{\Gamma} \sum_{a,\tilde{a}}\sum_{\x,\tilde{\x}}\left(\sum_{\s,\tilde{\s}} \Gamma_{\S,\tilde{\S} |\X,A ,\tilde{\X} ,\tilde{A}}\left( \s,\tilde{\s} |\x,a,\tilde{\x} ,\tilde{a}\right)  C_{\S, \tilde{\S}}(\s, \tilde{\s}) \right) \Gamma_{\X,A ,\tilde{\X} ,\tilde{A}}\left( \x ,a,\tilde{\x} ,\tilde{a}\right) \cdot\\
        & \;\;\;\;\;I\left( \x=\tilde{\x} ,a=\tilde{a}\right)\\ 
        & = \min_{\Gamma} \sum_{a}\sum_{\x}\left(\sum_{\s,\tilde{\s}} \Gamma_{\S,\tilde{\S} |\X,A ,\tilde{\X} ,\tilde{A}}\left( \s,\tilde{\s} |\x,a, \x , a\right)  C_{\S, \tilde{\S}}(\s, \tilde{\s}) \right) \Gamma_{\X,A ,\tilde{\X} ,\tilde{A}}( \x,a,\x,a)\\ 
        &\overset{(7)}{=}\min_{\Gamma}\sum_{a,\x} \langle \Gamma_{\S,\tilde{\S} |\X,A ,\tilde{\X} ,\tilde{A}} ,C_{\S, \tilde{\S}}\rangle \hat P( \x,a)\\ 
        &=OT_{X, A}^{con},
\end{align*}
\end{small}
where the seventh equality holds since all mass will be concentrated on the diagonal $a=\tilde{a}, \x=\tilde{\x}$. Next we prove the seventh equality. 

Considering the above minimization problem, let $\Gamma ^*$ be any optimal solutions, we next prove $\Gamma^*_{\X,A,\tilde \X,\tilde A}(\x,a,\x,a) = \hat P(\x,a)$:

Let $\Gamma^\dagger$ be an optimal plan. 
 Define the (off) diagonal 
\begin{align*}
\Gamma^{\dagger,\mathrm{diag}}(\s,\x,a,\tilde \s,\tilde \x,\tilde a)
:=\Gamma^\dagger(\s,\x,a,\tilde \s,\tilde \x,\tilde a)\,
\mathbf 1\{(\x,a)=(\tilde \x,\tilde a)\};
\\
\Gamma^{\dagger,\mathrm{offdiag}}(\s,\x,a,\tilde \s,\tilde \x,\tilde a)
:=\Gamma^\dagger(\s,\x,a,\tilde \s,\tilde \x,\tilde a)\,
\mathbf 1\{(\x,a)\neq (\tilde \x,\tilde a)\},
\end{align*}
where we have the off-diagonal mass $m$ as
\[
m:=\sum_{\s,\x,a,\tilde \s,\tilde \x,\tilde a}\Gamma^\dagger(\s,\x,a,\tilde \s,\tilde \x,\tilde a)\,
\mathbf 1\{(\x,a)\neq (\tilde \x,\tilde a)\}.
\]
We construct a coupling diagonal $\widetilde\Gamma$ by moving the off-diagonal mass of  $\Gamma^\dagger$ to its diagonal and reallocating the same amount of mass within each diagonal block $(\x,a)=(\tilde \x,\tilde a)$
so that the marginal constraints are preserved. Thai is
\[
\widetilde\Gamma:=\Gamma^{\dagger,\mathrm{diag}}+\Delta,
\quad \quad
\sum_{\s,\x,a,\tilde \s,\tilde \x,\tilde a}\Delta(\s,\x,a,\tilde \s,\tilde \x,\tilde a)=m,
\]
where $\Delta \geq 0$ is supported only on $(\x,a)=(\tilde \x,\tilde a)$.
Then we have
\begin{small}
    \begin{align*}
&\;\;\;\;\;\langle \Gamma^\dagger, C_{\S, \tilde{\S}}+C_{\X, \tilde{\X}}+C_{A, \tilde{A}}\rangle \\
&= \langle \Gamma^{\dagger,\mathrm{diag}}, C_{\S, \tilde{\S}} \rangle +\langle \Gamma^{\dagger,\mathrm{offdiag}}, C_{\S, \tilde{\S}} \rangle + \langle \Gamma^{\dagger,\mathrm{offdiag}}, C_{\X, \tilde{\X}} + C_{A, \tilde{A}} \rangle \\
& \geqslant \langle \Gamma^{\dagger,\mathrm{diag}}, C_{\S, \tilde{\S}} \rangle  + \langle \Gamma^{\dagger,\mathrm{offdiag}}, C_{\X, \tilde{\X}} + C_{A, \tilde{A}}  \rangle \\
& \geqslant \langle \Gamma^{\dagger,\mathrm{diag}}, C_{\S, \tilde{\S}} \rangle + \langle \Gamma^{\dagger,\mathrm{offdiag}}, \delta_{\X,A,\tilde \X,\tilde A} \rangle \\
& > \langle \Gamma^{\dagger,\mathrm{diag}}, C_{\S, \tilde{\S}}  \rangle + M \cdot \langle \Gamma^{\dagger,\mathrm{offdiag}}, 1 \rangle \\
& = \langle \Gamma^{\dagger,\mathrm{diag}}, C_{\S, \tilde{\S}}  \rangle + M \cdot m.
\end{align*}
\end{small}
where we use that fact that every element of $ \delta_{\X,A,\tilde \X,\tilde A}$ is larger than $M$, and $\sum \Gamma^{\dagger,\mathrm{offdiag}} = m $.

Since $\widetilde\Gamma$ is diagonal in $(X,A)$, we have $C_{\X, \tilde{\X}}+C_{A, \tilde{A}} = 0$ on its support, so:
\begin{small}
    \begin{align*}
\langle \widetilde \Gamma, C_{\S, \tilde{\S}}+C_{\X, \tilde{\X}}+C_{A, \tilde{A}}\rangle
&= \langle \Gamma^{\dagger,\mathrm{diag}}, C_{\S, \tilde{\S}} \rangle + \langle \Delta, C_{\S, \tilde{\S}} \rangle \\
& \leq \langle \Gamma^{\dagger,\mathrm{diag}}, C_{\S, \tilde{\S}} \rangle + M \langle \Delta, 1 \rangle \\
&= \langle \Gamma^{\dagger,\mathrm{diag}}, C_{\S, \tilde{\S}} \rangle + M \cdot m,
\end{align*}
\end{small}
where we used the fact that every element of \( C_s \) is smaller than $M$,  and \( \sum \Delta = m \).
Thus,
\[
\langle \Gamma^\dagger, C_{\S, \tilde{\S}}+C_{\X, \tilde{\X}}+C_{A, \tilde{A}}\rangle - \langle \widetilde \Gamma, C_{\S, \tilde{\S}}+C_{\X, \tilde{\X}}+C_{A, \tilde{A}}\rangle
> 0,
\]
which contradicts the optimality of \(\Gamma^\dagger\). Therefore, \(m = 0\), meaning that any optimal plan has no off-diagonal mass.

\end{proof}

\subsection{Proof of Theorem \ref{theorem: ot2}}
\label{proof: ot2}

\begin{proof}
    Assuming the probability measure of $\hat{P}(\S, \X, A|G=o)$ and $\hat{P}(\S, \X, A|G=e)$ are $\alpha, \beta$, 
    respectively.
    Then according to the definition, 
    the proof of Eq. (\ref{eq: mot_step1}) is equivalent to prove:
    \begin{small}
            \begin{align*}
    \left(\frac{1}{k}\sum _{i=1}^{k} \Gamma_{i}\right)\mathbf{1}_{n_e} =\alpha, ~~~
    \left(\frac{1}{k}\sum _{i=1}^{k} \Gamma _{i}\right)^{T}\mathbf{1}_{n_o}  =\beta.
    \end{align*}
    \end{small}

    Note that $\gamma _{i} \in \Pi(\hat{P}_{\mathcal{B}_i}(\S, \X, A|G=o), \hat{P}(\S, \X, A|G=e))$ satisfies $\gamma_i\mathbf{1}_{n_e}=\alpha_{\mathcal{B}_i}$ and $ \gamma_i^T\mathbf{1}_{b}=\beta$.
    Combining it with the definition of $\Gamma_i$, we have:
    \begin{align}
        \Gamma_i\mathbf{1}_{n_e}&=\bar{\alpha}_{\mathcal{B}_i}=\eta_i \odot \alpha, \label{eq: proof5}\\
        \Gamma_i^T\mathbf{1}_{n_o}&=\beta \label{eq: proof6},
    \end{align}
    where $\bar{\alpha}_{\mathcal{B}_i}\in \mathbb{R}^{n_o}$ is the extension of $\alpha_{\mathcal{B}_i}$ by padding zero entries to the dimension whose index does not belong to $\mathcal{B}_i$,
    $\odot $ corresponds to entry-wise product and $\eta_i$ is a $n_o$ dimensional vector with element satisfying that:
    \begin{align}
        \label{eq: proof7}
        \eta_i^j =\begin{cases}
        \frac{n_o}{b} =k, & if\ j\in \mathcal{B}_{i} ,\\
        0, & otherwise.
        \end{cases}
    \end{align}
    According Eq. (\ref{eq: proof5}) and (\ref{eq: proof6}), 
    we have:
    \begin{small}
            \begin{align}
        (\frac{1}{k}\sum _{i=1}^{k}\Gamma_i)\mathbf{1}_{n_e}
        &=\frac{1}{k}\sum _{i=1}^{k}(\Gamma_i\mathbf{1}_{n_e})
        =(\frac{1}{k}\sum _{i=1}^{k}\eta_i) \odot \alpha, \label{eq: proof8}\\
        (\frac{1}{k}\sum _{i=1}^{k}\Gamma_i)^T\mathbf{1}_{n_o}
        &=\frac{1}{k}\sum _{i=1}^{k}(\Gamma_i^T\mathbf{1}_{n_o})
        =\frac{1}{k}\sum _{i=1}^{k}\beta
        =\beta, \nonumber
    \end{align}
    \end{small}

    Combining Eq. (\ref{eq: proof7}) with conditions $\mathcal{B}_{i} \cap \mathcal{B}_{j} =\emptyset $ and $\cup _{i=1}^{k}\mathcal{B}_{i} =\{1,2,...,n_o\}$,
    we obtain that:
    \begin{align*}
        (\frac{1}{k}\sum _{i=1}^{k}\eta_i)^j=k\sum _{i=1}^{k}\mathbb{I}(j\in \mathcal{B}_i)=k,
    \end{align*}
    which means that the term $\frac{1}{k}\sum _{i=1}^{k}\eta_i$ in Eq. (\ref{eq: proof8}) is $\frac{1}{k}\sum _{i=1}^{k}\eta_i=\mathbf{1}_{n_o}$.
    As a result,
    we have already proved Eq. (\ref{eq: mot_step1}) because the following definition holds:
    \begin{small}
                \begin{align*}
    \left(\frac{1}{k}\sum _{i=1}^{k} \Gamma_{i}\right)\mathbf{1}_{n_e} =\alpha, ~~~
    \left(\frac{1}{k}\sum _{i=1}^{k} \Gamma _{i}\right)^{T}\mathbf{1}_{n_o}  =\beta.
    \end{align*}
    \end{small}

    By denoting the cost matrix between $\hat{P}(\S, \X, A|G=o)$ and $\hat{P}(\S, \X, A|G=e)$ as $C$,
    the above definition also leads to the following inequality:
    \begin{small}
            \begin{align}
        &\quad\;\;OT(\hat{P}(\S, \X, A|G=o), \hat{P}(\S, \X, A|G=e)) \nonumber\\
        &\leq \langle \frac{1}{k}\sum _{i=1}^{k} \Gamma_{i}, C\rangle
        = \frac{1}{k}\sum _{i=1}^{k} \langle  \Gamma_{i}, C\rangle \nonumber\\
        &= \frac{1}{k}\sum _{i=1}^{k} OT(\hat{P}_{\mathcal{B}_i}(\S, \X, A|G=o), \hat{P}(\S, \X, A|G=e)) \label{eq: proof9} \\
        &=  m\mbox{-}OT(\hat{P}(\S, \X, A|G=o), \hat{P}(\S, \X, A|G=e))\nonumber,
    \end{align}
    \end{small}
    where the Eq. (\ref{eq: proof9}) comes from the definition of $~\Gamma_i$.
    So far, we have finished the proof of Eq. (\ref{eq: mot_step1}).
\end{proof}

\subsection{Proof of Theorem \ref{theorem: bound}}
\label{proof: bound}

\begin{proof}
The inequality is equivalent to
\begin{small}
        \begin{align}
&\;\;\;\;\;\;\cE_{cf} -\cE_{f}^{\w_o} \nonumber\\
 &= \int _{\mathcal{X}}\int _{\mathcal{A}} P( \x) P( a)\ell_{\phi, g, h}(\x, a) \ d\x da
 - \int _{\mathcal{X}}\int _{\mathcal{A}} P^{\w_o}( \x,a)\ell_{\phi, g, h}(\x, a) \ d\x da \nonumber\\
 & =\int _{\mathcal{X}}\int _{\mathcal{A}}\left( P( \x) P( a) -P^{\w_o}( \x,a)\right)\ell_{\phi, g, h}(\x, a) \ d\x da \nonumber\\
 & =  \int _{\mathcal{Z}}\int _{\mathcal{A}}\left( P( \z) P( a) -P^{\w_o}( \z,a)\right)\ell_{\phi, g, h}(\psi ( \z), a) \ d\z da \label{eq: proof1}\\
 & \leqslant \sup _{m\in \mathcal{M}} \ \left| \int _{\mathcal{Z}}\int _{\mathcal{A}}\left( P( \z) P( a) -P^{\w_o}( \z,a)\right) m( \z,a) \ d\z da\right| \label{eq: proof2}\\
 & =IPM_{\mathcal{M}}\left( P( \Z) P( A) ,\ P^{\w_o}( \Z,A)\right)
\end{align}
\end{small}
where Eq. (\ref{eq: proof1}) can be obtained by the standard change of variables formula, 
using the determinant of the Jacobian of $\psi(\z)$, 
Eq. (\ref{eq: proof2}) is according to the definition of IPM with the assumption that $\frac{1}{B_{\Phi}}\cdot \ell_{\phi, g, h}(\x, a)\in \cM$.
\end{proof}

\subsection{Derivation about Projected Mirror Descent}
\label{app: optim}
In this section,
following \cite{yan2024exploiting},
we develop a projected mirror descent \cite{nemirovskij1983problem,raskutti2015information} based on the Kullback-Leibler (KL) divergence to solve the following problem,
which is non-trivial to address because of the equality constrints:
\begin{align*}
    &\min_{\w_{o, \mathcal{B}}} \min_{\gamma \in \Pi(\alpha_{\mathcal{B}}, \beta)}
    \langle \gamma, \C \rangle
    + \lambda_e \Omega(\gamma),\\
    s.t.~
    &\Pi(\alpha_{\mathcal{B}}, \beta)=\{\gamma\in \bR^{b\times n_e}|\gamma \mathbf{1}_{n_e}=\w_{o, \mathcal{B}}, \gamma^T \mathbf{1}_b=\bm{\mu}, \gamma_{ij}\in \left[0, 1\right]\},
\end{align*}
Firstly,
we rewrite the constrain about $\gamma$ as
\begin{align}
\label{eq: proof12}
    \Pi^{'}(\alpha_{\mathcal{B}}, \beta)=\{\gamma\in \bR^{b\times n_e}|\gamma^T \mathbf{1}_b=\bm{\mu}, \gamma_{ij}\in \left[0, 1\right]\},
\end{align}
which does not consider the constraint $\gamma \mathbf{1}_{n_e}=\w_{o,\mathcal{B}}$
since $\w_{o,\mathcal{B}}$ are also parameters to be optimized.
Based on this, 
the OT problem with a negative entropy regularization is given as follows:
\begin{align}
    \label{eq: proof10}
    \min_{\gamma}
    \langle \gamma, \C \rangle
    + \lambda_e \Omega(\gamma),~~~
    s.t.~
    \gamma\in \Pi^{'}(\alpha_{\mathcal{B}}, \beta).
\end{align}
For simplicity, 
we define the objective function in Problem (\ref{eq: proof10}) as
$
f(\gamma) = \langle \gamma, \C \rangle + \lambda_e \Omega(\gamma)$,
and the $(i,j)$-th element of the gradient $\nabla f(\gamma)$ is denoted by $\nabla_{ij}$,
\begin{align*}
\nabla_{ij} 
=  C_{ij} + \lambda_e \log w_{i}
=  C_{ij} + \lambda_e \log \gamma_{i\cdot},
\end{align*}
where $\gamma_{i\cdot}$ is the sum of $i$-th row of $\gamma$.
Then at each iteration,
we solve the following problem:
\begin{align}
    \label{eq: proof11}
    \gamma^{k} = \min_{\gamma \in \Pi^{'}(\alpha_{\mathcal{B}}, \beta)} 
    \eta \langle \nabla f(\gamma^{k-1}), \gamma \rangle + \mathcal{D}(\gamma || \gamma^{k-1}),
\end{align}
which firstly performs proximal gradient descent with the Bregman divergence and the stepsize $\eta$,
and then obtains a feasible solution in the set $\Pi^{'}(\alpha_{\mathcal{B}}, \beta)$ by projection.
Next, we present the details of these two operations.

\textbf{Proximal Gradient Descent} 
Let $\upsilon^{k}$ be the solution to Problem (\ref{eq: proof11}) without considering the constraint $\gamma \in \Pi^{'}(\alpha_{\mathcal{B}}, \beta)$,
i.e.,
\begin{align*}
\upsilon^{k} = \min_{\gamma} 
 ~~  \eta \langle \nabla f(\gamma^{k-1}), \gamma \rangle + \mathcal{D}(\gamma || \gamma^{k-1}).
\end{align*}
We adopt the KL divergence between two distributions $\gamma$ and $\gamma^{k-1}$
as the Bregman divergence $\mathcal{D}(\gamma || \gamma^{k})$,
then the closed-form solution to the above problem is given as:
\begin{align}
\label{eq: proof17}
\upsilon^{k} = \gamma^{k-1} \odot \exp( - \eta \nabla f(\gamma^{k-1})).
\end{align}

\textbf{Projection Operation}
To make sure $\gamma^{k}$ satisfies the constraints in Eq. (\ref{eq: proof12}),
we update $\gamma^{k}$ by finding $\gamma \in \Pi^{'}(\alpha_{\mathcal{B}}, \beta)$ which is most close to $\upsilon^{k}$
under the KL metric.
This is achieved by solving the projection problem (ignore constraint $\gamma_{ij}\in \left[0, 1\right]$ for now):
\begin{align}
\label{eq: proof16}
&\min_{\gamma}  \quad 
\mathcal{D} (\gamma || \upsilon^k)
= \sum_{i=1}^{b} \sum_{j=1}^{n_e} \gamma_{ij} \log(\frac{\gamma_{ij}}{\upsilon^k_{ij}}) - \gamma_{ij} + \upsilon^k_{ij} ,\nonumber\\
& s.t. ~~\gamma^{\top} \mathbf{1}_{b} =\bm{\mu}.
\end{align}
By introducing the Lagrangian multipliers $\bm{\lambda} = [\lambda_1, \ldots, \lambda_{n_t}]^{\top}$ for the constraint $\gamma^{\top} \mathbf{1}_{b} =\bm{\mu}$,
we obtain the Lagrangian $\mathcal{L}(\gamma, \bm{\lambda})$ as:
\begin{align*}
\mathcal{L}(\gamma, \bm{\lambda}) = 
\sum_{i=1}^{b} \sum_{j=1}^{n_e} \gamma_{ij} \log(\frac{\gamma_{ij}}{\upsilon^k_{ij}}) - \gamma_{ij} + \upsilon^k_{ij}
+ \bm{\lambda}^{\top} (\gamma^{\top} \mathbf{1}_{b} -\bm{\mu}).
\end{align*}
By taking the partial derivative of $\mathcal{L}(\gamma, \bm{\lambda})$ with respect to $\gamma_{ij}$ to zero,
we obtain:
\begin{align}
 \log \gamma_{ij} = \log \upsilon^k_{ij} - \lambda_j 
 \Rightarrow 
 \gamma_{ij} = \upsilon^k_{ij} \exp (-\lambda_j) \label{eq: proof13}.
\end{align}
According to the equality constraint $\gamma^{\top} \mathbf{1}_{b} =\bm{\mu}$,
we have $\sum_{i=1}^{b} \gamma_{ij} = 1/n_e$.
Combining it with the above result,
we further obtain:
\begin{align}
\sum_{i=1}^{b} \gamma_{ij} 
= 
\sum_{i=1}^{b} \upsilon^k_{ij} \exp (-\lambda_j)
& = 
\exp (-\lambda_j)
\sum_{i=1}^{b} \upsilon^k_{ij}
= \frac{1}{n_e}
\nonumber \\
\Rightarrow
\exp (-\lambda_j)
& = 
\frac{1}{n_e \sum_{i=1}^{b} \upsilon^k_{ij}}. \label{eq: proof14}
\end{align}
Combining Eq. (\ref{eq: proof13}) and (\ref{eq: proof14}),
we obtain the solution:
\begin{align}
\label{eq: proof15}
\gamma_{ij}  = 
\upsilon^k_{ij}
\Big /
\Big( 
n_e \sum_{i=1}^{b} \upsilon^k_{ij} 
\Big).
\end{align}
For an initial value $\gamma_{ij}^{0} \geq 0$,
given $\upsilon^k_{ij} > 0$ which is guaranteed by the update rule in Eq. (\ref{eq: proof12}),
it is obvious that the solution obtained by Eq. (\ref{eq: proof15})
satisfies the box constraint
$\gamma_{ij} \in [0,1]$.
Therefore, 
Problem (\ref{eq: proof16}) does not consider this constraint explicitly.

For now,
we have derived how to solve the Problem Eq. (\ref{eq: proof10}).
That is,
we repeat steps Eq. (\ref{eq: proof17}) and (\ref{eq: proof15}) until converging.
After that, 
the weights of observational samples could be obtained as $w_{i} = \sum_{j=1}^{n_e} \gamma^{*}_{ij}$,
where $w_i$ is the sum of the $i$-th row of $\gamma^{*}$.

\section*{GenAI Usage Disclosure}
We used generative AI tools at specific stages of this work. During manuscript preparation, OpenAI ChatGPT (GPT-5) was employed to revise phrasing and improve clarity.

\bibliographystyle{ACM-Reference-Format}
\bibliography{sample-base}

\end{document}